\documentclass{article}

\usepackage{microtype}
\usepackage{graphicx}
\usepackage{subcaption}
\usepackage{booktabs} 
\usepackage{enumitem}
\usepackage[most]{tcolorbox}
\usepackage{hyperref}
\usepackage{empheq}

\usepackage[preprint]{icml2026}

\usepackage{amsmath}
\usepackage{amssymb}
\usepackage{mathtools}
\usepackage{amsthm}

\usepackage[capitalize,noabbrev]{cleveref}

\usepackage{xcolor}

\newcommand{\bmark}[2]{%
  {\color{blue}\underbrace{\color{black}{#1}}_{\color{blue}\text{#2}}}%
}

\theoremstyle{plain}
\newtheorem{theorem}{Theorem}[section]

\newtheorem{lemma}[theorem]{Lemma}

\theoremstyle{definition}

\newtheorem{assumption}[theorem]{Assumption}
\theoremstyle{remark}
\newtheorem{remark}[theorem]{Remark}

\usepackage[textsize=tiny]{todonotes}

\newcommand{\EE}{\mathbb{E}}
\newcommand{\cF}{\mathcal{F}}
\newcommand{\cJ}{\mathcal{J}}
\newcommand{\cL}{\mathcal{L}}

\newcommand{\cO}{\mathcal{O}}
\newcommand{\RR}{\mathbb{R}}
\newcommand{\SSS}{\mathbb{S}}

\newcommand{\diag}{\operatorname{diag}}
\newcommand{\Polar}{\operatorname{Polar}}
\newcommand{\rank}{\operatorname{rank}}
\newcommand{\sym}{\operatorname{sym}}
\newcommand{\tr}{\operatorname{tr}}
\newcommand{\vect}{\operatorname{vec}}

\newcommand{\alg}{FISMO }

\icmltitlerunning{FISMO: Fisher-Structured Momentum-Orthogonalized Optimizer}

\begin{document}

\twocolumn[
  \icmltitle{FISMO: Fisher-Structured Momentum-Orthogonalized Optimizer}



  \icmlsetsymbol{equal}{*}

  \begin{icmlauthorlist}
    \icmlauthor{Chenrui Xu}{cuhkie}
    \icmlauthor{Wenjing Yan}{cuhkie}
    \icmlauthor{Ying-Jun Angela Zhang}{cuhkie}
  \end{icmlauthorlist}

  \icmlaffiliation{cuhkie}{Department of  Information Engineering, The Chinese University of Hong Kong, Hong Kong}
  \icmlcorrespondingauthor{Wenjing Yan}{wjyan@ie.cuhk.edu.hk}

  \icmlkeywords{Machine Learning, ICML}

  \vskip 0.3in
]



\printAffiliationsAndNotice{}  

\begin{abstract}
Training large-scale neural networks requires solving nonconvex optimization where the choice of optimizer fundamentally determines both convergence behavior and computational efficiency. While adaptive methods like Adam have long dominated practice, the recently proposed Muon optimizer achieves superior performance through orthogonalized momentum updates that enforce isotropic geometry with uniform singular values. However, this strict isotropy discards potentially valuable curvature information encoded in gradient spectra, motivating optimization methods that balance geometric structure with adaptivity. We introduce \textbf{FISMO} (Fisher-Structured Momentum-Orthogonalized) optimizer, which generalizes isotropic updates to incorporate anisotropic curvature information through Fisher information geometry. By reformulating the optimizer update as a trust-region problem constrained by a Kronecker-factored Fisher metric, FISMO achieves structured preconditioning that adapts to local loss landscape geometry while maintaining computational tractability. We establish convergence guarantees for FISMO in stochastic nonconvex settings, proving an $\mathcal{O}(1/\sqrt{T})$ rate for the expected squared gradient norm with explicit characterization of variance reduction through mini-batching. Empirical evaluation on image classification and language modeling benchmarks demonstrates that FISMO achieves superior training efficiency and final performance compared to established baselines.
\end{abstract}


\section{Introduction}

The remarkable progress of large language models (LLMs) \cite{touvron2023llama,achiam2023gpt,team2023gemini} has intensified the challenge of training neural networks with millions to billions of parameters. This scale necessitates solving highly nonconvex optimization problems, where the choice of optimizer fundamentally determines both computational efficiency and model performance. While Adam \cite{kingma2014adam} and its variant AdamW \cite{loshchilov2017decoupled} have served as the de facto standards for over a decade, the demands of modern deep learning continue to drive the search for more effective optimization algorithms.

A recent breakthrough in this pursuit is Muon \cite{jordan2024muon}, a matrix-parameter optimizer that departs radically from element-wise adaptive methods. By orthogonalizing the momentum matrix to normalize all singular values to unity, Muon produces isotropic updates that prevent pathological amplification or suppression along specific directions. This geometric approach has demonstrated substantial improvements in both training stability and downstream accuracy across diverse applications \cite{liu2025muon,wang2025muon,tveit2025muon}. In practice, Muon employs Newton--Schulz iterations \cite{higham2008functions} to efficiently approximate the orthogonalization operation.

Theoretically, Muon implements steepest descent under a spectral norm constraint \cite{bernstein2024old,chen2025muon}. Following \citet{bernstein2025deriving}, Muon's update solves the constrained linear minimization oracle (LMO) problem \cite{lan2013complexity}:
\begin{equation}
\label{eq:muon_lmo_spec}
\min_{\Delta W\in\mathbb{R}^{m\times n}} \langle \nabla_W \mathcal{L}, \Delta W \rangle
\quad \text{s.t.} \quad \|\Delta W\|_2 \leq \eta,
\end{equation}
where $\mathcal{L}$ denotes the objective function, $W$ represents the matrix parameters, and $\|\cdot\|_2$ denotes the spectral norm. This formulation seeks the update that maximally decreases the linearized objective while constraining the worst-case amplification: $\|\Delta W\|_2 = \max_{\|x\|_2=1}\|\Delta W x\|_2$. The solution aligns with the gradient's singular directions and saturates the spectral constraint, yielding an isotropic update with uniform singular values \cite{chen2025muon,bernstein2025deriving}.

However, the optimality of isotropic geometry for deep learning remains contested. The heterogeneous singular-value spectrum of gradient matrices encodes valuable information about curvature and scale variations across directions—information that uniform normalization may discard \cite{lau2025polargrad}. Furthermore, theoretical analysis under isotropic-curvature models suggests that while Muon excels when curvature exhibits rapid transitions, such conditions may not characterize the loss landscapes of large neural networks \cite{su2025isotropic}. Therefore, the isotropic update of Muon might not be the optimal choice in complicated nonconvex settings, especially in large-scale deep learning. These considerations raise a fundamental question:

\begin{tcolorbox}[colback=gray!5,
    colframe=black!70, 
    frame style={line width=10pt}, arc=2mm]
\textbf{Beyond isotropic updates, what additional structure should an optimizer leverage to produce more informative and effective update directions for training deep networks?}
\end{tcolorbox}

To address this question, we develop a structured-geometry framework for optimizer design. We reformulate optimizer updates through the LMO lens \cite{lan2013complexity} and establish connections to natural-gradient methods \cite{martens2020new} under Fisher information geometry \cite{costa2015fisher}. Recognizing that exact Fisher information is computationally intractable, we derive the optimal Kronecker-structured approximation \cite{martens2015optimizing} and solve the resulting Fisher-structured LMO in closed form. This solution motivates a practical optimizer that balances computational efficiency with geometric informativeness, for which we establish convergence guarantees and demonstrate empirical effectiveness.

\subsection{Main Contributions}

Our main contributions are summarized below.

\begin{itemize}
    \item \textbf{Structured-Geometry Framework.} We develop a principled approach that bridges second-order methods and LMO-based optimizers by replacing uniform spectral constraints with Fisher information-induced trust regions. This framework reveals how local geometry shapes update directions, enabling adaptive, anisotropic updates that preserve curvature information lost in isotropic orthogonalization.
    
    \item \textbf{The FISMO Algorithm.} We introduce FISMO (\textbf{FI}sher-\textbf{S}tructured \textbf{M}omentum-\textbf{O}rthogonalized) optimizer, a practical instantiation of our theoretical framework. FISMO combines Kronecker-factored Fisher preconditioning with orthogonalized momentum, achieving computational efficiency while adapting to local curvature. This design synthesizes the stability of orthogonalization with the informativeness of second-order geometry.
    
    \item \textbf{Convergence theory.} We establish an $\mathcal{O}(1/\sqrt{T})$ convergence rate for FISMO's expected squared gradient norm in stochastic nonconvex optimization, matching the standard nonconvex rate achieved by Muon. Our analysis characterizes the variance reduction achieved through mini-batching, providing theoretical guidance for practical hyperparameter selection.
    
    \item \textbf{Empirical Validation.} Through experiments on representative deep learning benchmarks, we demonstrate FISMO's improved training stability and performance relative to established optimizers. 

\end{itemize}

\paragraph{Notation.} We use lowercase letters for vectors and uppercase letters for matrices. $I_d$ denotes the $d \times d$ identity matrix and $\tr(\cdot)$ is the trace operator. For matrices $A, B \in \mathbb{R}^{m \times n}$, the Frobenius inner product is $\langle A, B \rangle_F := \tr(A^\top B)$. 
We denote $\|\cdot\|_F$ for the Frobenius norm (the $\ell_2$ norm of the entries), $\|\cdot\|_2$ (or $\|\cdot\|_{\mathrm{op}}$) for the operator/spectral norm (the largest singular value), and $\|\cdot\|_*$ for the nuclear norm (the sum of singular values). The vectorization operator $\vect(\cdot)$ stacks the columns of a matrix into a vector, and $A \otimes B$ denotes the Kronecker product of matrices $A$ and $B$. For a square matrix $X$, $\sym(X) := \frac{1}{2}(X+X^\top)$ is its symmetric part. For a matrix $M$ with SVD $M=U\Sigma V^\top$, its orthogonal polar factor is $\Polar(M) := UV^\top$. For symmetric matrices $A, B$, we write $A \succ B$ (or $A \succeq B$) if the matrix $A-B$ is positive definite (or positive semi-definite).

\section{Related Work}

\paragraph{Matrix-Parameter Optimizers.}
Matrix-parameter optimizers treat key network weights (e.g., attention and linear layers) as matrices and design update rules that explicitly exploit their row/column or block structure. Among these approaches, some works mainly aim to reduce Adam-style state and improve efficiency by sharing or factorizing second-moment statistics at the block or layer level, including Adafactor \cite{shazeer2018adafactor} and Adam-mini \cite{zhang2024adam}.
Other methods, such as Shampoo \cite{gupta2018shampoo} and its variant SOAP \cite{vyas2024soap}, exploit tractable matrix factorizations (like Kronecker structure) to approximate large curvature preconditioners.
More recently, orthogonalized matrix updates have emerged as a prominent design choice, with Muon \cite{jordan2024muon} as a representative example and an expanding line of work that extends this matrix-geometry-driven update shaping. Closely related to Muon, ASGO \cite{an2025asgo} utilizes a single-sided structured preconditioner built from accumulated matrix-gradient second moments. 
PolarGrad \cite{lau2025polargrad} retains Muon’s orthogonalized direction but adds a trace-based, curvature-dependent scaling, rather than enforcing a purely isotropic update.
MuonBP \cite{khaled2025muonbp} adapts Muon’s updates to large-scale parallel training via block-periodic orthogonalization.

\paragraph{Natural Gradient Methods.}
Natural gradient methods view learning as steepest descent under the Fisher information metric \cite{martens2020new}. 
In machine-learning settings, the Fisher matrix coincides with the generalized Gauss–Newton matrix \cite{shrestha2023natural}, so natural gradient updates can be interpreted as estimates of Hessian. Since forming the full Fisher is infeasible at modern scales, practical algorithms therefore rely on structured approximations.
For example, K-FAC \cite{martens2015optimizing} uses a Kronecker-factored approximation of Fisher blocks, with extensions including recurrent-network variants and eigenbasis corrections such as EKFAC \cite{george2018fast}.
Other approaches pursue lighter-weight or implicit approximations, such as quasi-diagonal Riemannian constructions like TONGA \cite{roux2007topmoumoute}. Some works like TANGO \cite{ollivier2017true} also asymptotically exact trajectories without explicit Fisher estimation.

\section{Preliminaries}

Modern neural network training requires minimizing a loss function $\mathcal{L}(\theta)$ over high-dimensional parameter spaces $\theta \in \mathbb{R}^d$, where $d$ ranges from millions to billions of parameters:
\begin{equation*}
\min_{\theta \in \mathbb{R}^d} \mathcal{L}(\theta) = \mathbb{E}_{(x,y) \sim \mathcal{D}} \left[ \ell(f_\theta(x), y) \right],
\end{equation*}
where $f_\theta$ denotes the neural network, $\ell$ represents the mini-batch loss function, and \(\mathcal{D}\) is the distribution of the data $z=(x,y)$. Predominant optimization methods employ first-order updates of the form $\theta_{t+1} = \theta_t - \eta_t \nabla \mathcal{L}(\theta_t)$ or adaptive variants thereof \cite{klein2009adaptive}.

A substantial portion of these parameters, particularly in transformer architectures, consists of matrix-valued weights $W \in \mathbb{R}^{m \times n}$. Standard optimizers, such as Adam, typically vectorize these matrices, treating them as unstructured arrays. This vectorization, however, discards critical geometric structures, including: low-rank gradient properties, coupling between input and output spaces, and spectral characteristics that govern optimization dynamics in high dimensions. Recent evidence suggests that preserving matrix geometry can accelerate convergence, improve conditioning, and enhance training stability \cite{gupta2018shampoo}.

\paragraph{The Muon Optimizer and Isotropic Geometry.} The Muon optimizer \cite{jordan2024muon} represents a recent advance in this direction, preserving matrix geometry through polar decomposition. For matrix parameters $W \in \mathbb{R}^{m \times n}$, Muon maintains a momentum accumulator and orthogonalizes it as:
\begin{align*}
M_t &= \beta M_{t-1} + (1-\beta) \nabla_W \mathcal{L}(W_{t-1}), \\
W_t &= W_{t-1} - \eta_t \text{Polar}(M_t),
\end{align*}
where \(\beta \in [0,1)\) is the momentum coefficient and \(\text{Polar}(M_t) = UV^\top\) is the orthogonal factor from the SVD, \(M_t = U\Sigma V^\top\).

This update mechanism has an elegant interpretation as the solution to a Linear Minimization Oracle (LMO) subproblem with a spectral norm constraint:
\begin{equation*}
\min_{\|\Delta W\|_2 \leq \eta} \langle M_t, \Delta W \rangle_F.
\end{equation*}
The solution, \(\Delta W^* = -\eta \cdot \text{Polar}(M_t)\), normalizes the singular values of the momentum matrix to unity while preserving the singular vectors. Consequently, it enforces a perfectly \textbf{isotropic} geometry: all update directions are scaled equally, irrespective of the curvature information encoded in the gradient's singular value spectrum. Computationally, Muon avoids the cost of a full SVD by approximating the polar decomposition using Newton-Schulz iterations, which rely only on efficient matrix multiplications. Within months of its release, Muon has been adopted across diverse training scenarios, from language model pretraining to vision transformers, consistently demonstrating 1.5-2× faster convergence than AdamW while maintaining or improving final performance \cite{jordan2024muon,liu2025muon}.

However, this isotropic constraint discards the heterogeneous curvature information encoded in the gradient's singular value spectrum—High-curvature directions (corresponding to large singular values) are treated identically to low-curvature ones (small singular values), potentially limiting adaptation to local landscape geometry.

\paragraph{Natural Gradient Descent and Anisotropic Geometry.}
In contrast, Natural Gradient Descent (NGD) offers a principled framework for adapting to the local geometry of the loss landscape \cite{martens2015optimizing}. NGD preconditions the gradient with the Fisher Information Matrix (FIM), which serves as the canonical Riemannian metric on the statistical manifold of the model's parameters \cite{ly2017tutorial}:
\begin{equation*}
F(\theta) = \mathbb{E}_{z \sim p_\theta} \left[\nabla_\theta \log p_\theta(z) \nabla_\theta \log p_\theta(z)^\top\right].
\end{equation*}
The NGD is updated by solving a FIM-regularized trust-region problem:
\begin{equation}
\begin{aligned}
    \label{eq:fisher-prob}
    &\min_{\Delta\theta}~ \langle \nabla\mathcal{L}(\theta_t), \Delta\theta\rangle \\
    &~~{\rm s.t.} ~~ \Delta\theta^\top F(\theta_t) \Delta\theta \le \eta^2,
\end{aligned}    
\end{equation}
which yields the update \(\Delta\theta^* \propto -F(\theta_t)^{-1}\nabla\mathcal{L}(\theta_t)\). This \textbf{anisotropic} update adapts to local curvature—taking larger steps along flat directions and smaller steps along steep ones. In stark contrast to Muon's isotropic updates, NGD thus exploits the heterogeneous geometric structure of the parameter space.

The practical utility of NGD, however, is severely limited by the prohibitive cost of computing, storing, and inverting the \(d \times d\) FIM for large-scale models. This establishes a fundamental dichotomy in modern optimization:  Muon achieves efficiency through isotropic orthogonalization but sacrifices curvature information, while NGD captures geometric structure at prohibitive cost. Our work bridges this gap by incorporating FIM-derived geometry into an LMO framework, thereby achieving geometry-aware updates with the computational efficiency of Muon.

\section{Theoretical Framework}

Building upon the NGD principles, we now develop our approach for optimizing matrix-valued parameters $W \in \mathbb{R}^{m \times n}$. Our goal is to derive a geometry-aware update rule that preserves the benefits of NGD while remaining computationally tractable for large-scale models.

\subsection{Kronecker-Factored Approximation of FIM}

Direct manipulation of the empirical FIM—an $(mn) \times (mn)$ matrix—is computationally prohibitive for large-scale neural networks. We therefore adopt the K-FAC strategy \cite{martens2015optimizing} to approximate the FIM through a Kronecker product. Specifically, we introduce two symmetric positive definite matrices, $P \in \SSS_{++}^m$ and $Q \in \SSS_{++}^n$, such that:
\begin{equation*}
    F_W \leftarrow Q \otimes P.
\end{equation*}
Under this approximation, the natural gradient trust region constraint in \eqref{eq:fisher-prob} becomes:
\begin{equation*}
    \operatorname{vec}(\Delta W)^\top (Q \otimes P) \operatorname{vec}(\Delta W) \leq \eta^2,
\end{equation*}
where $\operatorname{vec}(\Delta W)$ denotes the vectorization of the matrix update. Using the identity $\operatorname{vec}(\Delta W)^\top (Q \otimes P) \operatorname{vec}(\Delta W) = \operatorname{tr}(\Delta W^\top P\Delta WQ)$, we reformulate this constraint as:
\begin{equation*}
    \operatorname{tr}(\Delta W^\top P \Delta W Q) = \|P^{1/2}\Delta W Q^{1/2}\|_F^2 \leq \eta^2.
\end{equation*}
We replace the Frobenius norm with the spectral norm based on their relationship $\|X\|_2 \leq \|X\|_F \leq \sqrt{\operatorname{rank}(X)}\|X\|_2$ \cite{horn2012matrix}. This substitution preserves the geometric structure while aligning our formulation with the Muon method \cite{jordan2024muon}, which employs a spectral norm constraint (see \eqref{eq:muon_lmo_spec}).

This leads to our central optimization problem, which seeks the steepest descent direction $\Delta W$ within a trust region defined by the Kronecker-factored geometry:
\begin{empheq}[box=\fbox]{equation}
\begin{aligned}
    \label{eq:lmo}
    \min_{\Delta W \in \mathbb{R}^{m \times n}} &~\langle G, \Delta W \rangle_F \\
    {\rm s.t.}\quad &\| P^{1/2} \Delta W Q^{1/2} \|_2 \leq \eta,
\end{aligned}
\end{empheq}

where $G:=\nabla_W\cL(W)$ denotes the gradient of the loss function. This formulation raises two fundamental questions:
\begin{enumerate}[label={[Q\arabic*]}, ref={[Q\arabic*]}]
    \item What choices of $P$ and $Q$ best approximate the empirical Fisher information matrix $F_W$?
    \label{Q:1}
    \item What is the optimal solution to \eqref{eq:lmo} that determines our update rules? 
    \label{Q:2}
\end{enumerate}

\subsection{The Best Kronecker Approximation}
\label{sec:best-kronecker-approx}

To answer question \ref{Q:1}, we need to first define what constitutes the ``best'' Kronecker-product approximation. We measure the discrepancy between two positive definite matrices using the \emph{log-det divergence} \cite{cichocki2015log}:
\begin{equation*}
\label{eq:logdet_div}
D_{\mathrm{ld}}(A\,\|\,B) \coloneqq \tr(B^{-1}A) - \log\det(B^{-1}A) - d,
\end{equation*}
which is equivalent (up to constants) to the Kullback--Leibler (KL) divergence between two centered multivariate Gaussian \(\mathcal{N}(0, A)\) and \(\mathcal{N}(0, B)\). We seek the preconditioners \(P\) and \(Q\) that minimize this divergence between the damped empirical FIM and its Kronecker approximation that:
\begin{equation}
    \label{eq:best_kron_problem}
\min_{P\in\SSS_{++}^m,\ Q\in\SSS_{++}^n}
\;D_{\mathrm{ld}}\left(F_W + \mu I_{mn}\ \big\|\ Q\otimes P\right).
\end{equation}
Here, \(\mu > 0\) is a small damping constant that ensures the matrix is strictly positive definite, a common practice in natural gradient methods. Using the properties \((Q\otimes P)^{-1}=Q^{-1}\otimes P^{-1}\) and \(\log\det(Q\otimes P)=n\log\det P+m\log\det Q\), Problem \eqref{eq:best_kron_problem} is equivalent to minimizing the objective \(\cJ(P,Q)\) defined as:
\begin{equation}
\label{eq:best_kron_obj_matrix}
    \begin{aligned}
        \cJ(P,&Q)
        := \tr\left((Q^{-1}\!\otimes P^{-1})(F_W+\mu I_{mn})\right)\\
        &\;+\; n\log\det P \;+\; m\log\det Q.\\
        =&\mathbb{E}\!\left[\mathrm{tr}\!\left(P^{-1}GQ^{-1}G^\top\right)\right]
        +\mu\,\mathrm{tr}(P^{-1})\,\mathrm{tr}(Q^{-1})\\
        &\;+\; n\log\det P \;+\; m\log\det Q.
    \end{aligned}
\end{equation}
The optimal preconditioners \(P\) and \(Q\) are characterized by the following theorem, which suggests an alternating minimization scheme for their computation. 
\begin{theorem}[Optimal Kronecker Approximation]
\label{thm:opt-Kron-approx}
    Suppose \(\EE\|G\|_F < \infty\). The objective function \(\cJ(P,Q)\) in \eqref{eq:best_kron_obj_matrix} is convex in \(P\) (for fixed \(Q\)) and in \(Q\) (for fixed \(P\)). The unique minimizers are coupled via the following fixed-point equations:
    \begin{itemize}
        \item For any fixed \(Q \in \SSS_{++}^n\), the unique minimizer over \(P \in \SSS_{++}^m\) is:
        \[
        P^*(Q)=\frac{1}{n}\EE[GQ^{-1}G^\top] + \frac{\mu\tr(Q^{-1})}{n} I_m.
        \]
        \item For any fixed \(P \in \SSS_{++}^m\), the unique minimizer over \(Q \in \SSS_{++}^n\) is:
        \[
        Q^*(P)=\frac{1}{m}\EE[G^\top P^{-1}G] + \frac{\mu\tr(P^{-1})}{m} I_n.
        \]
    \end{itemize}
\end{theorem}
The proof of \cref{thm:opt-Kron-approx} is provided in Appendix~\ref{app:best-pq-thm}.

\subsection{The Best Update}

With the best choice of $P$ and $Q$, we now answer the question~\ref{Q:2} to find the optimal solution for Problem \eqref{eq:lmo}. 
Applying a whitening change of variables to the objective, we discover a lower bound of \eqref{eq:lmo} by von Neumann trace inequality \ref{lem:tr-ineq}.
By constructing a feasible update that attains this bound, we obtain the following theorem, whose proof is placed in Appendix~\ref{app:update-thm}.

\begin{theorem}
    \label{thm:update}
    Suppose $G \in \RR^{m \times n}$ is a nonzero matrix, $P \in \SSS^{m}_{++}, Q \in \SSS^{n}_{++}$ are defined as in Theorem~\ref{thm:opt-Kron-approx}. Let $\eta>0$ be a given radius. Denote $\widetilde{G}:=P^{-1/2}GQ^{-1/2}$, and suppose $\widetilde{G}=U\Sigma V$ is the singular value decomposition of $\widetilde{G}$. Then the optimal solution to \eqref{eq:lmo} is
        \begin{align*}
            \Delta W^* &= -\eta P^{-1/2}UV^\top Q^{-1/2} \\
            &= -\eta P^{-1/2}\Polar(\widetilde{G})Q^{-1/2}.
        \end{align*}
    Moreover, the optimal objective value is 
    $$\min_{\|P^{1/2}\Delta W Q^{1/2}\|_2\le \eta} \langle G,\Delta W\rangle_F = -\eta\|\widetilde{G}\|_*.$$
 \end{theorem}

\begin{remark}
    Theorem~\ref{thm:update} gives a closed-form expression for the best update direction under the $(P,Q)$-preconditioned spectral-norm trust region. In particular, the optimal step is characterized by the polar factor $\Polar(\widetilde{G})=UV^\top$ of the preconditioned gradient $\widetilde{G}=P^{-1/2}GQ^{-1/2}$. These characterizations will be used as the main building block for the algorithmic construction in the next section.
\end{remark}

\begin{algorithm}[tb]
  \caption{\alg}
  \label{alg:kl}
  \begin{algorithmic}[1]
    \STATE {\bfseries Input:} $W_0 \in \RR^{m \times n}$; learning rate $\eta$;
    momentum $\beta$; EMA $\gamma$; damping factor $\mu$; $P_0\leftarrow I_m$; $Q_0\leftarrow I_n$; $M_0=0$
    \FOR{$t=1, \cdots , T$}
      \STATE Compute gradient $G_t = \frac{1}{B}\sum_{i=1}^B \nabla_W \ell(W_{t-1};\xi_{t,i})$
      \STATE {Update left preconditioner}\\
       \qquad $L_t \leftarrow \frac{1}{n}\,G_t\,Q_{t-1}^{-1}\,G_t^\top + \mu \frac{\mathrm{tr}(P_{t-1})}{m} I_m$\\
       \qquad $\widetilde{P}_t \leftarrow \gamma P_{t-1} + (1-\gamma)L_t$\\
       \qquad $P_t \leftarrow \sym\left(\frac{m}{\tr(\widetilde{P}_t)}\,\widetilde{P}_t \right)$\\
      \STATE {Update right preconditioner}\\
       \qquad $R_t \leftarrow \frac{1}{m}\,G_t^\top\,P_t^{-1}\,G_t + \mu \frac{\mathrm{tr}(Q_{t-1})}{n} I_n$\\
       \qquad $\widetilde{Q}_t \leftarrow \gamma Q_{t-1} + (1-\gamma)R_t$\\
       \qquad $Q_t \leftarrow \sym\left(\frac{n}{\tr(\widetilde{Q}_t)}\,\widetilde{Q}_t \right)$
      \STATE {Whiten gradient} $\widetilde{G}_t \leftarrow P_t^{-1/2}\,G_t\,Q_t^{-1/2}$
      \STATE $M_t \leftarrow \beta M_{t-1} + (1-\beta)\widetilde{G}_t$
      \STATE $\Delta W_t \leftarrow P_t^{-1/2}\,\mathrm{Polar}(M_t)\,Q_t^{-1/2}$
      \STATE {Update} $W_{t} \leftarrow W_{t-1} - \eta\,\Delta W_t$
    \ENDFOR
  \end{algorithmic}
\end{algorithm}

\section{ {\alg} Algorithm}
\subsection{Algorithm Design}
\label{subsec:alg-design}

Based on the results of \cref{thm:opt-Kron-approx} and \cref{thm:update}, we develop the Fisher-Information Structured Momentum Orthogonalization (FISMO) algorithm, presented in Algorithm~\ref{alg:kl}. The algorithm iteratively (i) updates the left and right preconditioners $(P_t,Q_t)$, (ii) whitens the gradient accordingly, and (iii) computes the optimal update in the whitened space. We now describe the practical implementation details.

\paragraph{Gauss--Seidel Preconditioner Updates.}
Theorem~\ref{thm:opt-Kron-approx} establishes that the optimal left preconditioner $P$ depends on the right factor $Q$, and vice versa. Consequently, $(P,Q)$ must be updated sequentially rather than simultaneously. We employ Gauss--Seidel iteration: as shown in Steps~4--5 of Algorithm~\ref{alg:kl}, we compute $P_t$ using $Q_{t-1}$, then compute $Q_t$ using the newly obtained $P_t$.

This sequential updating offers advantages over Jacobi-style iteration, which computes both factors from $(P_{t-1},Q_{t-1})$ simultaneously. First, the Gauss--Seidel method achieves superior conditioning by incorporating the updated left geometry when computing $Q_t$, thereby accelerating convergence of the whitening transform $P_t^{-1/2}(\cdot)Q_t^{-1/2}$. Second, it reduces memory requirements: after computing $P_t$, only the current value is retained, eliminating storage of $P_{t-1}$.

\paragraph{Preconditioner Stabilization.}
The optimal preconditioner $(P,Q)$ given in \cref{thm:opt-Kron-approx} are derived from instantaneous curvature estimates that may exhibit high variance in stochastic settings. To mitigate this, we employ three stabilization techniques in Algorithm~\ref{alg:kl}.

\emph{1) Exponential Moving Average.} Rather than directly setting $P_t\leftarrow L_t$ and $Q_t\leftarrow R_t$, we apply Exponential Moving Average (EMA) smoothing that:
$$
\widetilde{P}_t=\gamma P_{t-1}+(1-\gamma)L_t,
\quad
\widetilde{Q}_t=\gamma Q_{t-1}+(1-\gamma)R_t,
$$
where $\gamma \in [0,1]$ is the decay factor controlling the trade-off between stability and responsiveness. This recursive filtering scheme assigns exponentially decaying weights to historical observations, with recent estimates receiving higher weights than older ones. Specifically, the contribution of an observation from $k$ steps ago is weighted by $\gamma^k(1-\gamma)$, ensuring smooth adaptation to changing curvature while suppressing high-frequency noise from mini-batch sampling.

\emph{2) Identity Regularization.} The matrices $L_t$ and $R_t$ incorporate scaled-identity terms (e.g., $\mu \tr(P_{t-1})/m \cdot I_m$ for $L_t$) to maintain a minimum eigenvalue threshold. This regularization arises naturally from our theoretical framework in Section~\ref{sec:best-kronecker-approx}. In problem~\eqref{eq:best_kron_problem}, the damping factor $\mu$ stabilizes the Kronecker approximation; Theorem~\ref{thm:opt-Kron-approx} yields solutions that explicitly contain this $\mu I$ component. This ensures that preconditioners remain positive definite even when empirical factors are rank-deficient, thereby preventing numerical instability in directions with insufficient gradient information. The trace-based scaling factors $\tr(P_{t-1})/m$ and $\tr(Q_{t-1})/n$ maintain dimensional consistency with the current preconditioner magnitudes.

\emph{3) Trace Normalization.} Algorithm~\ref{alg:kl} symmetrizes and normalizes the EMA estimates that:
$$P_t \leftarrow \text{sym}\left(\frac{m}{\tr(\widetilde{P}_t)}\widetilde{P}_t\right), \quad Q_t \leftarrow \text{sym}\left(\frac{n}{\tr(\widetilde{Q}_t)}\widetilde{Q}_t\right).$$
This eliminates scale non-identifiability in the bilinear metric, since $(P,Q)$ and $(cP,c^{-1}Q)$ induce identical preconditioned geometries for any $c>0$. Fixing $\tr(P_t)=m$ and $\tr(Q_t)=n$ prevents scale drift between factors and ensures numerical stability. Combined with $\mu$-damping from Theorem~\ref{thm:opt-Kron-approx}, this normalization yields uniform bounds on $\|P_t^{-1/2}\|_2$ and $\|Q_t^{-1/2}\|_2$, which is essential for convergence analysis.

\paragraph{Structured Orthogonalized Momentum.}
Theorem~\ref{thm:update} establishes that the optimal update for Problem~\eqref{eq:lmo} requires an orthogonal transformation in preconditioned coordinates. Accordingly, FISMO maintains momentum in the whitened space and applies orthogonalization to determine the update direction.

After whitening the stochastic gradient using current preconditioners (Step 6, Algorithm~\ref{alg:kl}), we update the momentum buffer in the whitened space (Step 7, Algorithm~\ref{alg:kl}). Rather than using $\widetilde{G}_t$ directly, we compute the polar decomposition of the smoothed momentum, $\text{Polar}(M_t)$. This orthogonalization procedure combines exponential averaging for variance reduction with the structured direction prescribed by Theorem~\ref{thm:update} within the $(P_t,Q_t)$-geometry. The resulting update direction is transformed back to the original coordinates via
$$\Delta W_t = P_t^{-1/2}\text{Polar}(M_t)Q_t^{-1/2},$$
as derived in Theorem~\ref{thm:update}.

Since exact polar decomposition via SVD is computationally prohibitive at scale, we follow Muon \cite{jordan2024muon} and approximate $\text{Polar}(M_t)$ using Newton--Schulz iteration \cite{higham2008functions}. This iterative scheme requires only matrix multiplications, enabling efficient implementation on modern accelerator architectures.

\begin{remark}
In practice, exact gradient computation is often intractable, necessitating stochastic approximations. We therefore approximate the gradient matrix $G$ in \eqref{eq:lmo} by a mini-batch average in our algorithm implementation:
$G_t = \frac{1}{B}\sum_{i=1}^{B}\nabla_W \ell(W_{t-1};\xi_{t,i})$,
where $\mathcal{B}_t=\{\xi_{t,1},\ldots,\xi_{t,B}\}$ denotes a mini-batch of size $B$ with samples drawn i.i.d.\ from the data distribution $\mathcal{D}$. 
\end{remark}

\subsection{Convergence Analysis}

This subsection establishes convergence guarantees for Algorithm~\ref{alg:kl} in stochastic nonconvex optimization. Our analysis relies on the following standard assumptions.
\begin{assumption}[Lower-Bounded Objective]
The objective function $\cL$ is lower-bounded with an inferior of $\cL_*$ that:
    \label{ass:LB}
    $$\cL_*:=\inf_W\cL(W)>-\infty.$$
\end{assumption}
\begin{assumption}[$L$-Smoothness]
\label{ass:L-smooth}
The funciton $\cL(W)$ is Lipschitz smooth in sense of Frobenius norm, i.e., there exist a constant $L$ such that, for all $W_1,W_2 \in \RR^{m\times n}$, 
    $$\|\nabla\cL(W_1)-\nabla\cL(W_2)\|_F\le L\|W_1-W_2\|_F.$$
\end{assumption}
\begin{assumption}[Unbiased Stochastic Gradients]
\label{ass:unbias-grad}
    For all iteration $t$ and any random sample $\xi$, the stochastic gradient is unbiased that: $$\EE[\nabla \ell(W_{t-1}; \xi)|\cF_{t-1}]=\nabla \cL(W_{t-1}),$$
    where $\mathcal{F}_{t-1}$ is the $\sigma$-algebra generated by all random variables revealed up to iteration $t-1$.
\end{assumption}
\begin{assumption}[Bounded Variance]
    \label{ass:bv}
    For all iteration $t$ and any random sample $\xi$, there exists $\sigma^2<\infty$ such that $$\mathbb{E} \left[ \left\| \nabla \ell(W_{t-1}; \xi) - \nabla \mathcal{L}(W_{t-1}) \right\|_F^2 \,\middle|\, \cF_{t-1} \right] \leq \sigma^2.$$
\end{assumption}
\begin{assumption}[Bounded Gradient]
\label{ass:B-grad}
The stochastic gradient is bounded at each iteration $t$, i.e.,
there is a constant $G_*$ such that
$\|G_t\|_* \le G_*$ holds for all $t$.
\end{assumption}

Then we have the following lemma, whose proof can be found in Appendix~\ref{app:cvg-lemma}.
\begin{lemma} 
    \label{lem:cvg-lem}
    Suppose that Assumptions~\ref{ass:LB}--\ref{ass:bv} hold. Denote $K_{PQ}(t) = \|P_t^{-1/2}\|_2 \cdot\|Q_t^{-1/2}\|_2$. Then for each iteration $t\ge1$ in Algorithm~\ref{alg:kl}, the following one-step upper bound for $\cL(W_t)$ holds:
\begin{align}
    \label{eq:cvg-lem-final}
    \cL(W_t) &\leq \cL(W_{t-1}) 
    - \eta K_{PQ}(t)\| \nabla \mathcal{L}(W_{t-1}) \|_*\nonumber\\
    &\;+ 2\eta K_{PQ}(t)\| \nabla \mathcal{L}(W_{t-1}) - G_t \|_*\nonumber \\
    &\;+ 2\eta \| \widetilde{G}_t - M_t \|_* + \frac{L\eta^2}{2}\min(m,n)K_{PQ}^2(t). \nonumber
\end{align}
\end{lemma}

We aggregate Lemma~\ref{lem:cvg-lem} over iterations to obtain the main convergence theorem. The detailed proof of the theorem is deferred to Appendix~\ref{app:cvg-thm}.

\begin{theorem}[Main convergence theorem]
\label{thm:main_convergence_kl_muon_bigO}
Suppose the Assumptions~\ref{ass:LB}--\ref{ass:bv} holds. Define $R:=\cL(W_0)-\cL_\star$ with $\cL_\star:=\inf_W \cL(W)$, and let $r:=\min(m,n)$. Suppose the stochastic gradient magnitude is bounded by $\|G_t\|_*\le G_*$ for all $t$. Run Algorithm~\ref{alg:kl} for $T$ iterations with stepsize $\eta=\frac{C}{\sqrt{T}}$ for some constant $C>0$. Then the iterates $\{W_t\}_{t=0}^{T-1}$ satisfy:
\begin{equation*}
\label{eq:main_convergence_bigO}
\begin{aligned}
&\frac{1}{T}\sum_{t=1}^T \EE\left[\|\nabla \cL(W_{t-1})\|_*\right]\\
&=\left(\frac{R + rL + G_*}{\sqrt{T}} +\frac{G_*}{T}+\frac{\sigma\sqrt{r}}{\sqrt{B}}\right).
\end{aligned}
\end{equation*}

\end{theorem}

\paragraph{Remarks.}
The bound in \eqref{eq:main_convergence_bigO} separates the optimization error from the stochastic sampling error.
In particular, the term $\cO(\sigma\sqrt{r}/\sqrt{B})$ is unavoidable when using mini-batch stochastic gradients: it captures the intrinsic variance of the gradient estimator and can only be reduced by increasing the batch size $B$.
If we ignore this stochastic term (e.g., in the large-batch or deterministic regime), the remaining terms imply an $\cO(1/\sqrt{T})$ convergence rate with respect to the number of iterations $T$, matching the standard nonconvex rate achieved by Muon (up to problem- and dimension-dependent constants) \cite{shen2025convergence, li2025note}.

\section{Experimental Results}

To evaluate the performance of the proposed \alg optimizer, we conduct experiments on two standard benchmarks: (i) the OpenWebText \cite{Gokaslan2019OpenWeb} corpus for language modeling, and (ii) the CIFAR-10 \cite{krizhevsky2009learning} dataset for image classification. For the language modeling task, we utilize the NanoGPT \cite{karpathy2022nanogpt} framework to reproduce a 124M-parameter GPT-2 \cite{radford2019language} model. For CIFAR-10, we employ the SimpleDLA \cite{yu2018deep} architecture. 
We compare \alg with several established and state-of-the-art baseline optimizers, including widely adopted methods like SGD \cite{robbins1951stochastic} and AdamW \cite{loshchilov2017decoupled}, as well as advanced algorithms such as Shampoo \cite{gupta2018shampoo} and Muon \cite{jordan2024muon}.

\begin{figure}[ht]
  \centering

  \begin{subfigure}[t]{\columnwidth}
    \centering
    \includegraphics[width=0.95\linewidth]{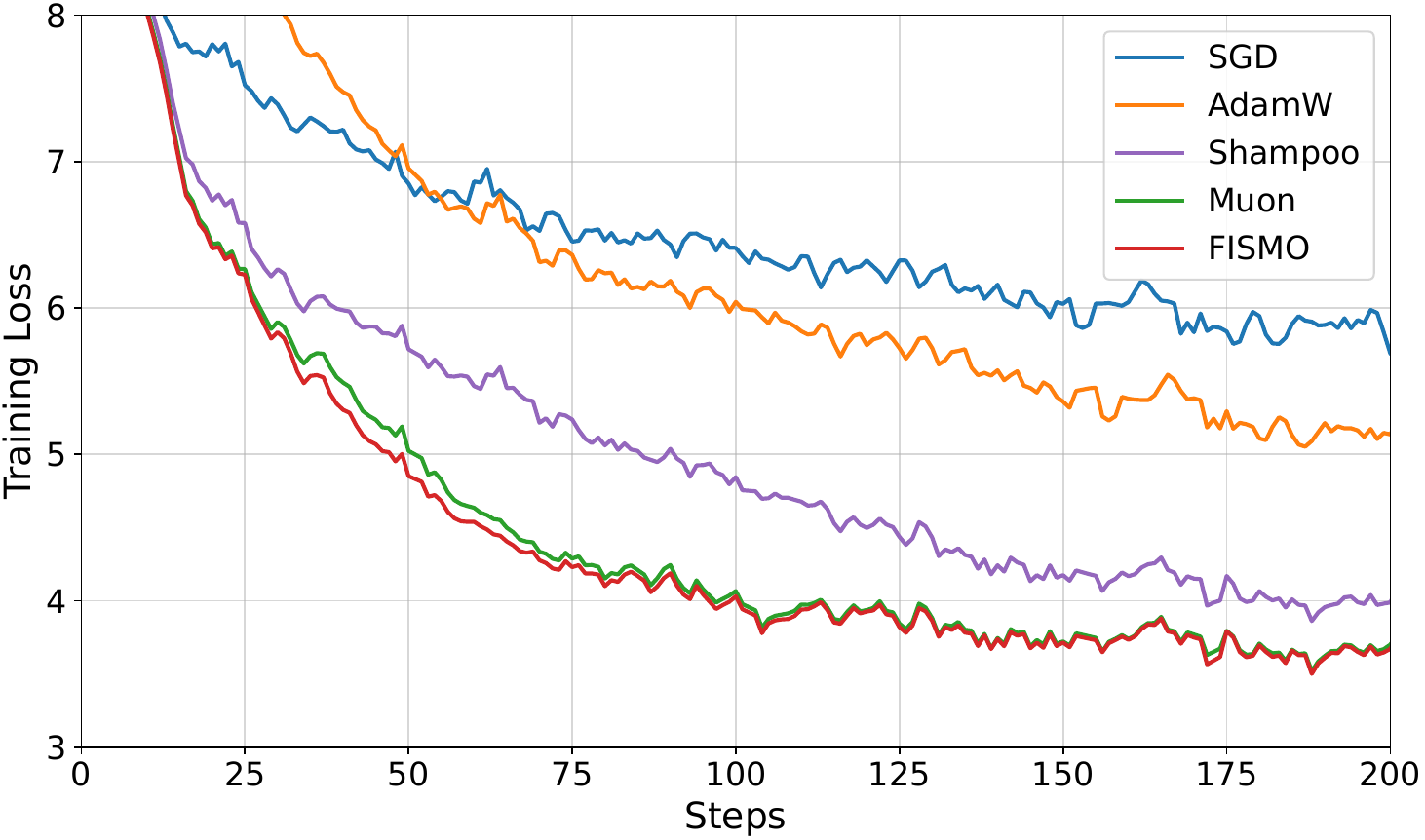}
    \caption{Training loss}
    \label{fig:r1-train}
  \end{subfigure}
  \begin{subfigure}[t]{\columnwidth}
    \centering
    \includegraphics[width=0.95\linewidth]{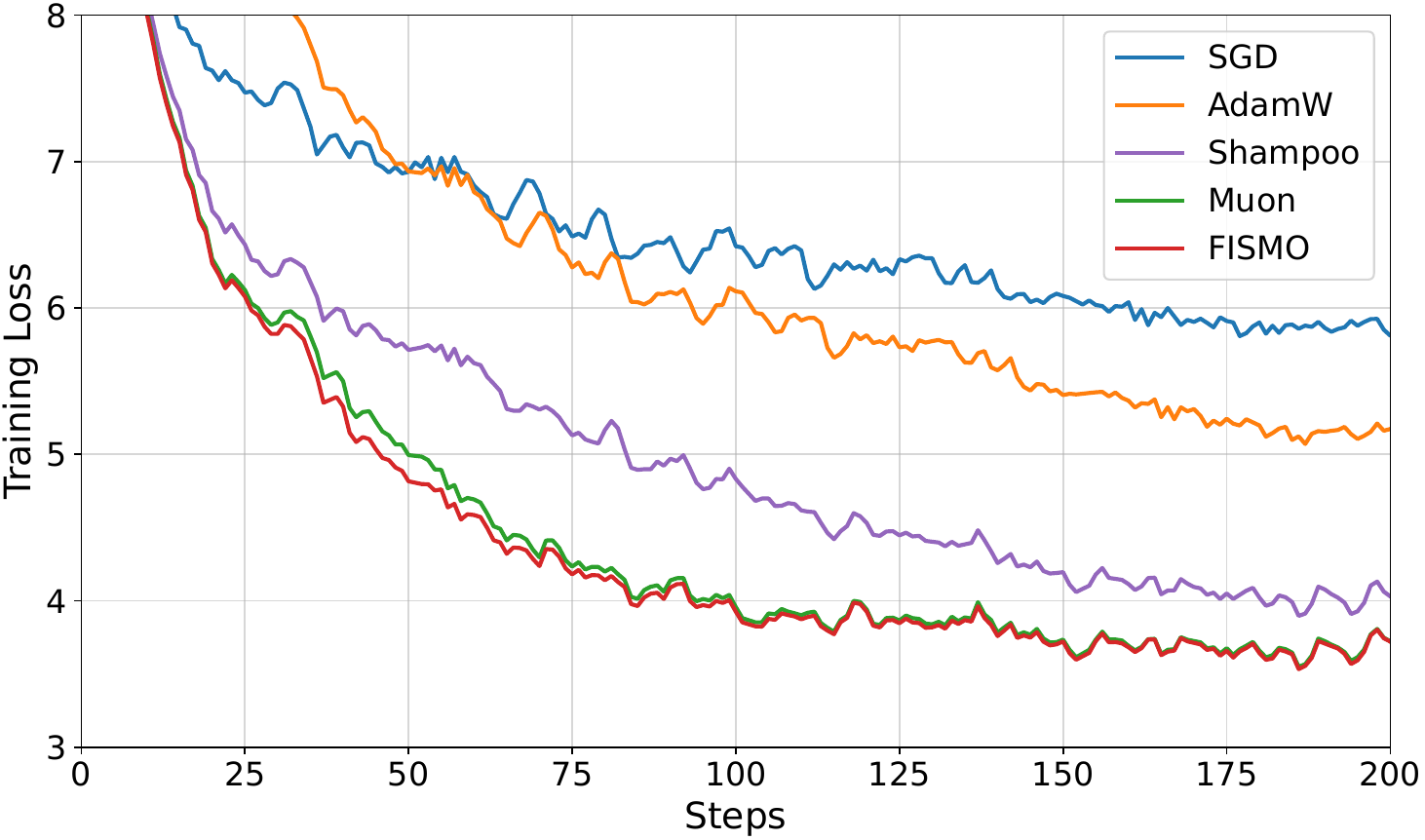}
    \caption{Validation loss}
    \label{fig:r1-val}
  \end{subfigure}

  \caption{Training and validation loss versus the number of training steps on the OpenWebText dataset (nanoGPT).}
  \label{fig-r1}
\end{figure}

Figure~\ref{fig-r1} reports the training (\ref{fig:r1-train}) and validation (\ref{fig:r1-val}) loss versus training steps for the 124M-parameter GPT-2 model trained on OpenWebText under the NanoGPT framework, comparing FISMO with baselines SGD, AdamW, Shampoo and Muon. Across the whole training horizon, \alg achieves the fastest loss reduction and the lowest final loss over the entire trajectory, surpassing even the strongest baseline Muon. The consistent ordering between training and validation suggests that the advantage of FISMO is not due to aggressive fitting but reflects more effective and stable optimization dynamics, aligning with FISMO’s update mechanism that preconditions matrix gradients under a Fisher-inspired geometry while avoiding the overly isotropic spectrum imposed by strict orthogonalization.

Figure~\ref{fig:cifar} illustrates the comparison results of FISMO and the baselines on the CIFAR-10 image classification task.
The results show that our method FISMO maintains a clear advantage across all four metrics: it drives the training loss down more rapidly, sustains the lowest testing loss over the full trajectory, and consequently achieves the highest training and testing accuracies throughout training. 
Moreover, compared with Muon, we observe that FISMO yields a noticeably smoother validation trajectory: both the testing loss (\ref{fig:four-valloss}) decreases and the testing accuracy (\ref{fig:four-valacc}) increases with fewer oscillations across training. A plausible explanation is that FISMO preserves informative curvature variation in matrix updates, instead of enforcing an aggressively isotropic step. Therefore, it avoids over-correction and improving the stability of generalization-oriented progress.

\begin{figure*}[!t]
  \centering

  \begin{subfigure}[t]{0.24\textwidth}
    \centering
    \includegraphics[width=\linewidth]{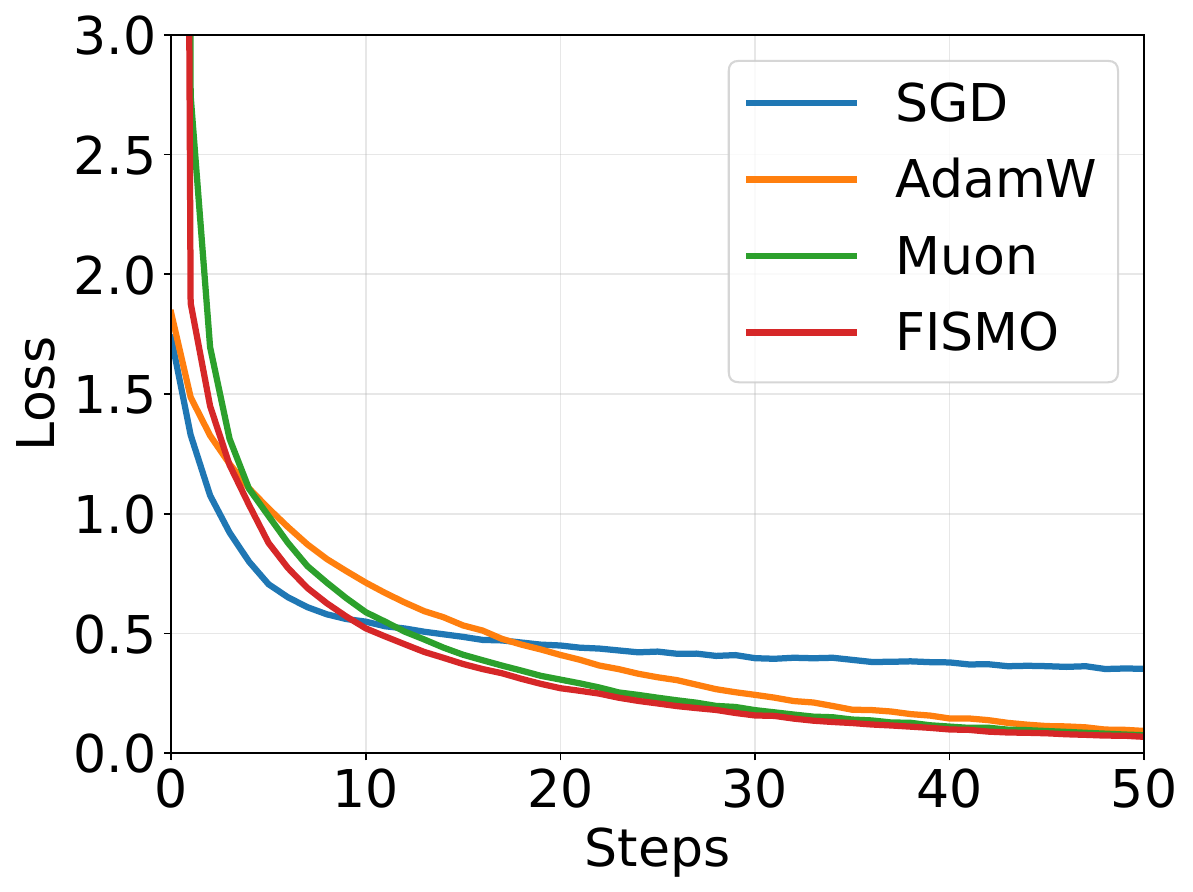}
    \caption{Training loss}
    \label{fig:four-trainloss}
  \end{subfigure}
  \hfill
  \begin{subfigure}[t]{0.24\textwidth}
    \centering
    \includegraphics[width=\linewidth]{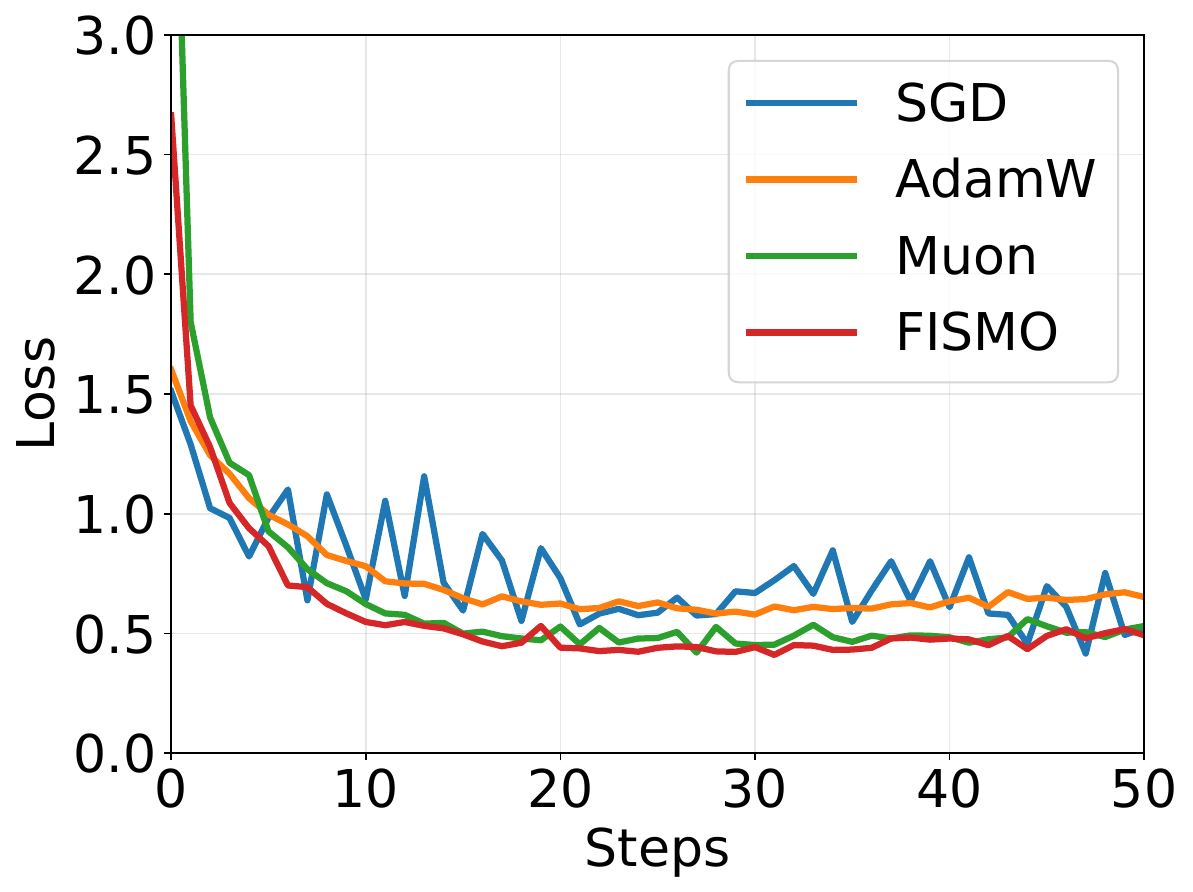}
    \caption{Validation loss}
    \label{fig:four-valloss}
  \end{subfigure}
  \begin{subfigure}[t]{0.24\textwidth}
    \centering
    \includegraphics[width=\linewidth]{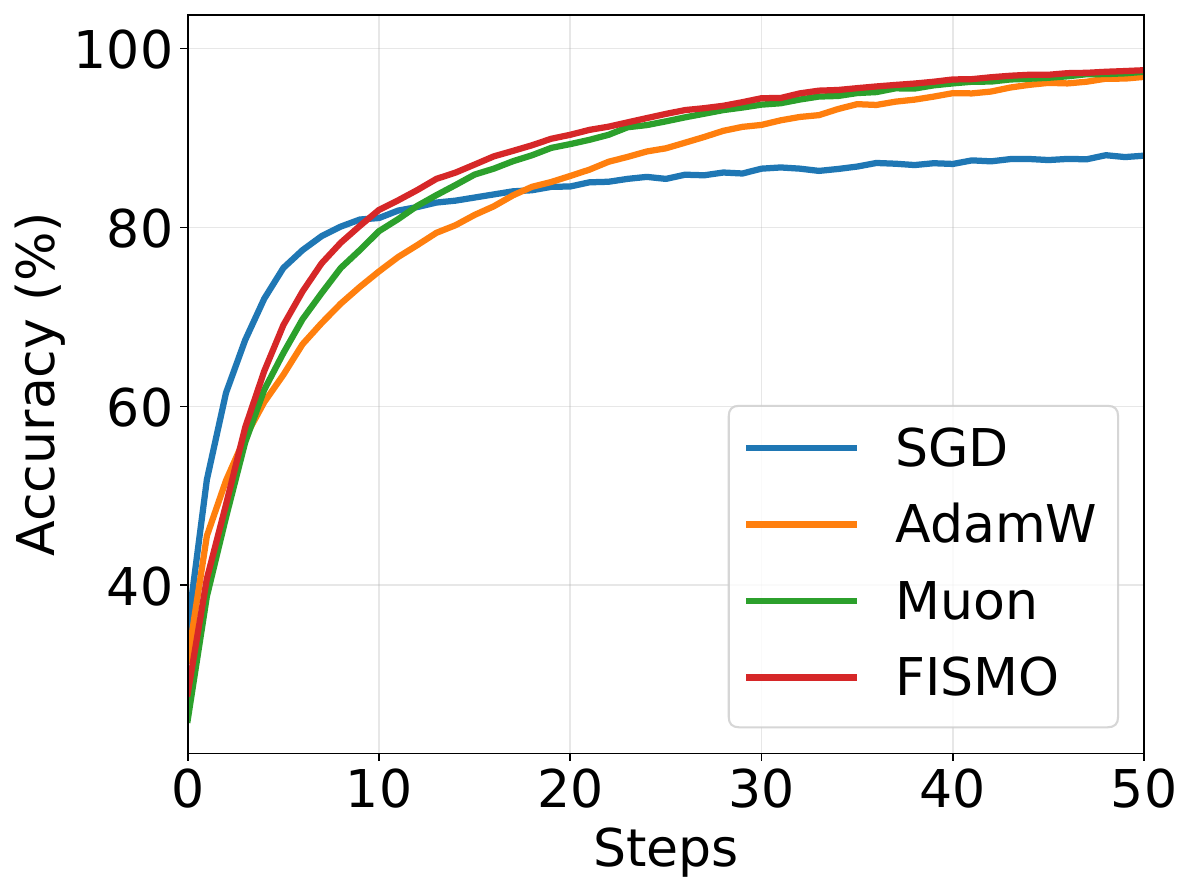}
    \caption{Training accuracy}
    \label{fig:four-trainacc}
  \end{subfigure}
  \hfill
  \begin{subfigure}[t]{0.24\textwidth}
    \centering
    \includegraphics[width=\linewidth]{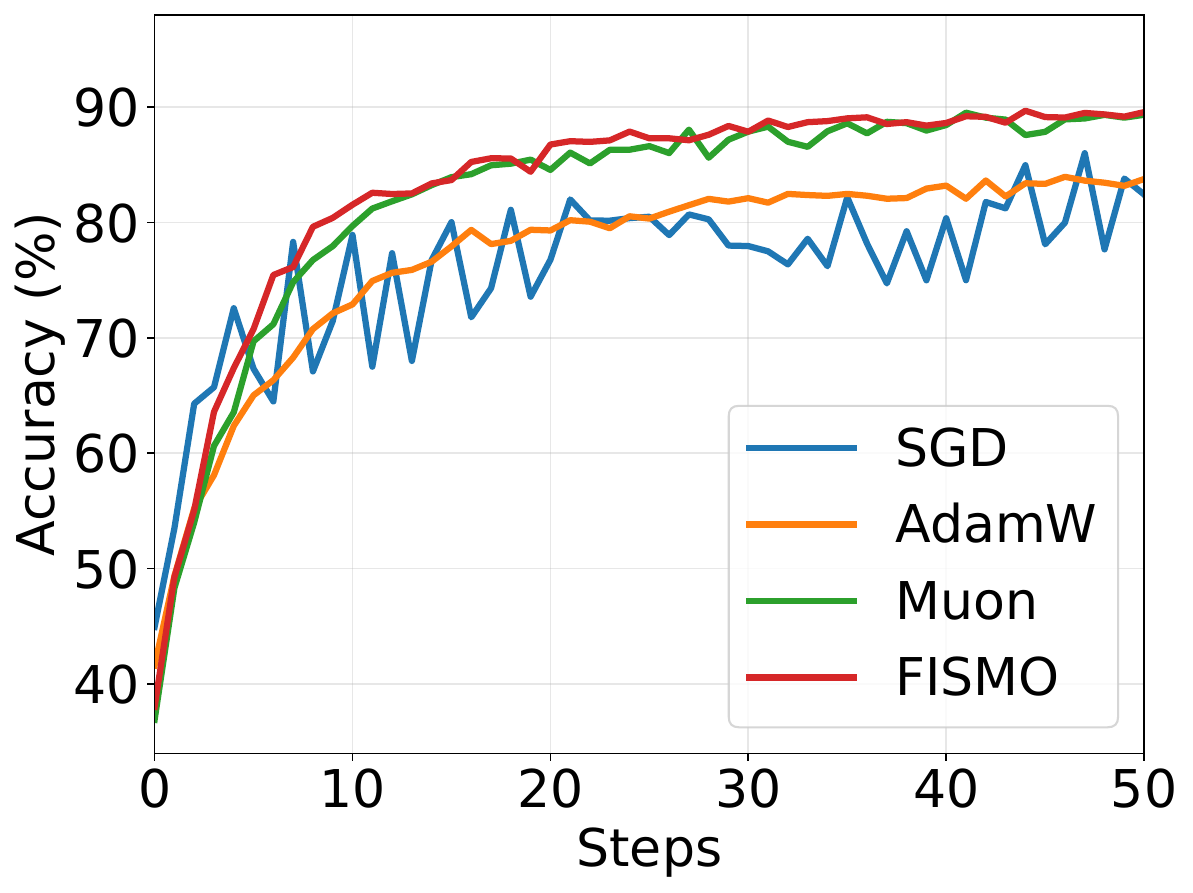}
    \caption{Validation accuracy}
    \label{fig:four-valacc}
  \end{subfigure}

  \caption{Traning/validation loss/accuracy versus the number of steps on the CIFAR-10 dataset.}
  \label{fig:cifar}
\end{figure*}

\section{Further Discussion on FISMO: A Condition Number Perspective}

To further understand the geometric efficacy of FISMO, we analyze the spectral properties of the update matrices, specifically focusing on the condition number. The condition number of a matrix (denoted as $\kappa$) is defined as the ratio of its greatest and lowest singular value, where a high condition indicates a highly anisotropic spectrum dominated by a few principal directions \cite{saad2003iterative}. Standard adaptive methods like Adam rely on element-wise scaling, often resulting in ill-conditioned updates with unconstrained global spectra. 
In contrast, Muon enforces strict orthogonality to achieve an isotropic update with an approximate condition number of 1, treating all spectral directions equally \cite{lau2025polargrad}.

As illustrated in Figure~\ref{fig:con_num}, we track the average condition number of the update matrices throughout training, comparing FISMO against Adam and Muon with Newton-Schulz (NS) iteration numbers of 5 (NS=5) and 7 (NS=7). Note that increasing NS iterations in Muon will enforce stricter orthogonalization, driving the update towards the ideal isotropic limit ($\kappa \to 1$). Generally, we observe the following relation of condition number among those optimizers.
\begin{equation*}
    \kappa_{\text{Adam}} \gg \kappa_{\text{FISMO}} > \kappa_{\text{Muon Practical}} > \kappa_{\text{Muon Ideal}} = 1.
\end{equation*}
As observed, Adam exhibits pathologically high condition numbers ($\kappa > 10^8$). In contrast, increasing Muon's NS iterations progressively flattens the spectrum, with NS=7 approaching the ideal isotropic limit. Crucially, FISMO establishes a stable spectral profile in the range of $10^2$--$10^3$. This represents a significant reduction in ill-conditioning compared to Adam (by several orders of magnitude) while maintaining a condition number generally higher than the orthogonalized updates from Muon.

We characterize this regime as an \textit{optimal conditioning trade-off}: FISMO effectively mitigates severe spectral skewness to ensure stability, yet retains a degree of anisotropy necessary to preserve relative curvature information that strict orthogonalization would eliminate. 

These spectral properties perfectly align with the recently proposed \textit{Isotropic Curvature Model} \cite{su2025isotropic}, which challenges the optimality of strict gradient orthogonalization. \citeauthor{su2025isotropic} demonstrates that under realistic super-quadratic curvature growth, the theoretically optimal update requires (partial) spectrum homogenization: the singular values should be compressed to be closer in magnitude to improve conditioning, but \textit{not necessarily forced to be completely equal}. 
Specifically, the optimal spectral transformation is rarely perfectly uniform ($\kappa=1$), because practical loss landscapes typically do not exhibit the extreme asymptotic behavior required to justify strict orthogonalization.
This implies that there may exist update mechanisms superior to Muon that operate with a condition number theoretically larger than 1.
By maintaining a condition number that is low but bounded away from unity, \alg effectively implements partial spectrum homogenization, navigating the middle ground between the chaotic anisotropy of element-wise methods and the rigid isotropy of fully orthogonalized updates.

\begin{figure}[t]
  \begin{center}    \centerline{\includegraphics[width=\columnwidth]{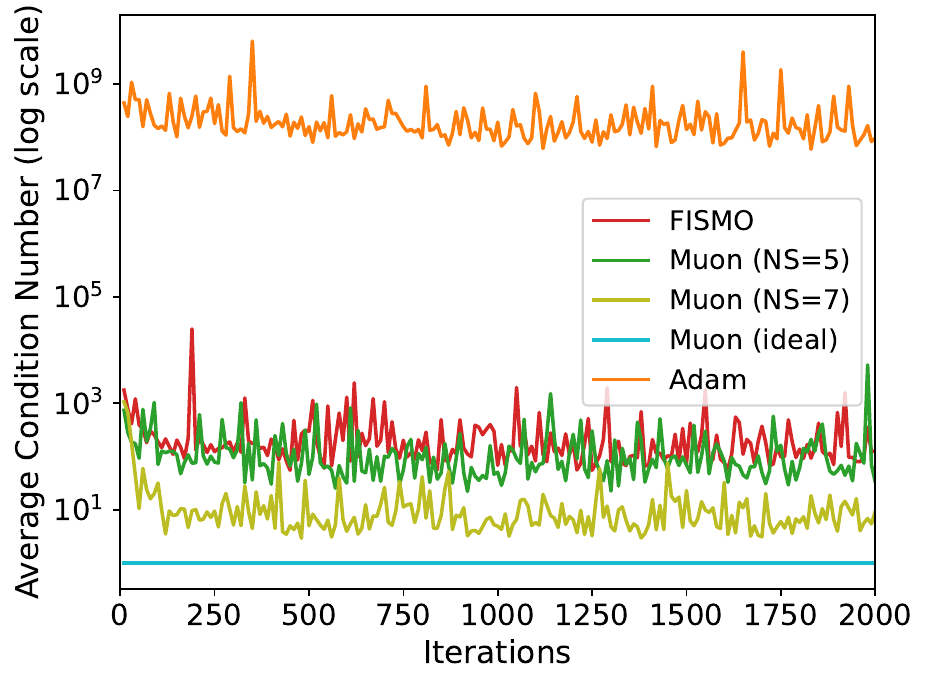}}
    \caption{Average condition number of the update matrices over training iterations. FISMO is compared with Adam and Muon (with 5 and 7 Newton-Schulz iterations), alongside the ideal isotropic Muon (strictly orthogonalized update).}
    \label{fig:con_num}
  \end{center}
\end{figure}

\section{Conclusions}
We proposed \textbf{FISMO}, a Fisher-structured momentum-orthogonalized optimizer for large-scale training with matrix-parameter updates.
By casting update as a trust-region problem under a Kronecker-factored Fisher metric, FISMO yields a structured preconditioning update rule that remains tractable while preserving informative spectral variation beyond strict isotropy. 
Through rigorous analysis, we prove that FISMO achieves an $\mathcal{O}(1/\sqrt{T})$ convergence guarantee for stochastic optimization.
Extensive experiments validate the efficiency and stability of FISMO.




\section*{Impact Statement}
This paper presents work whose goal is to advance the field
of Machine Learning. There are many potential societal
consequences of our work, none of which we feel must be
specifically highlighted here.


\bibliography{main}

\begin{thebibliography}{43}
\providecommand{\natexlab}[1]{#1}
\providecommand{\url}[1]{\texttt{#1}}
\expandafter\ifx\csname urlstyle\endcsname\relax
  \providecommand{\doi}[1]{doi: #1}\else
  \providecommand{\doi}{doi: \begingroup \urlstyle{rm}\Url}\fi

\bibitem[Achiam et~al.(2023)Achiam, Adler, Agarwal, Ahmad, Akkaya, Aleman, Almeida, Altenschmidt, Altman, Anadkat, et~al.]{achiam2023gpt}
Achiam, J., Adler, S., Agarwal, S., Ahmad, L., Akkaya, I., Aleman, F.~L., Almeida, D., Altenschmidt, J., Altman, S., Anadkat, S., et~al.
\newblock Gpt-4 technical report.
\newblock \emph{arXiv preprint arXiv:2303.08774}, 2023.

\bibitem[An et~al.(2025)An, Liu, Pan, Ren, Ma, Goldfarb, and Zhang]{an2025asgo}
An, K., Liu, Y., Pan, R., Ren, Y., Ma, S., Goldfarb, D., and Zhang, T.
\newblock Asgo: Adaptive structured gradient optimization.
\newblock \emph{arXiv preprint arXiv:2503.20762}, 2025.

\bibitem[Bernstein(2025)]{bernstein2025deriving}
Bernstein, J.
\newblock Deriving muon, 2025.
\newblock URL \url{https://jeremybernste.in/writing/deriving-muon}.

\bibitem[Bernstein \& Newhouse(2024)Bernstein and Newhouse]{bernstein2024old}
Bernstein, J. and Newhouse, L.
\newblock Old optimizer, new norm: An anthology.
\newblock \emph{arXiv preprint arXiv:2409.20325}, 2024.

\bibitem[Chen et~al.(2025)Chen, Li, and Liu]{chen2025muon}
Chen, L., Li, J., and Liu, Q.
\newblock Muon optimizes under spectral norm constraints.
\newblock \emph{arXiv preprint arXiv:2506.15054}, 2025.

\bibitem[Cichocki et~al.(2015)Cichocki, Cruces, and Amari]{cichocki2015log}
Cichocki, A., Cruces, S., and Amari, S.-i.
\newblock Log-determinant divergences revisited: Alpha-beta and gamma log-det divergences.
\newblock \emph{Entropy}, 17\penalty0 (5):\penalty0 2988--3034, 2015.

\bibitem[Costa et~al.(2015)Costa, Santos, and Strapasson]{costa2015fisher}
Costa, S.~I., Santos, S.~A., and Strapasson, J.~E.
\newblock Fisher information distance: A geometrical reading.
\newblock \emph{Discrete Applied Mathematics}, 197:\penalty0 59--69, 2015.

\bibitem[George et~al.(2018)George, Laurent, Bouthillier, Ballas, and Vincent]{george2018fast}
George, T., Laurent, C., Bouthillier, X., Ballas, N., and Vincent, P.
\newblock Fast approximate natural gradient descent in a kronecker factored eigenbasis.
\newblock \emph{Advances in neural information processing systems}, 31, 2018.

\bibitem[Gokaslan \& Cohen(2019)Gokaslan and Cohen]{Gokaslan2019OpenWeb}
Gokaslan, A. and Cohen, V.
\newblock Openwebtext corpus.
\newblock \url{http://Skylion007.github.io/OpenWebTextCorpus}, 2019.

\bibitem[Gupta et~al.(2018)Gupta, Koren, and Singer]{gupta2018shampoo}
Gupta, V., Koren, T., and Singer, Y.
\newblock Shampoo: Preconditioned stochastic tensor optimization.
\newblock In \emph{International Conference on Machine Learning}, pp.\  1842--1850. PMLR, 2018.

\bibitem[Higham(2008)]{higham2008functions}
Higham, N.~J.
\newblock \emph{Functions of matrices: theory and computation}.
\newblock SIAM, 2008.

\bibitem[Horn \& Johnson(1994)Horn and Johnson]{horn1994topics}
Horn, R.~A. and Johnson, C.~R.
\newblock \emph{Topics in matrix analysis}.
\newblock Cambridge university press, 1994.

\bibitem[Horn \& Johnson(2012)Horn and Johnson]{horn2012matrix}
Horn, R.~A. and Johnson, C.~R.
\newblock \emph{Matrix analysis}.
\newblock Cambridge university press, 2012.

\bibitem[Jordan et~al.(2024)Jordan, Jin, Boza, You, Cesista, Newhouse, and Bernstein]{jordan2024muon}
Jordan, K., Jin, Y., Boza, V., You, J., Cesista, F., Newhouse, L., and Bernstein, J.
\newblock Muon: An optimizer for hidden layers in neural networks, 2024.
\newblock URL \url{https://kellerjordan.github.io/posts/muon/}.

\bibitem[Karpathy(2022)]{karpathy2022nanogpt}
Karpathy, A.
\newblock nano{GPT}.
\newblock \url{https://github.com/karpathy/nanoGPT}, 2022.

\bibitem[Khaled et~al.(2025)Khaled, Ozkara, Yu, Hong, and Park]{khaled2025muonbp}
Khaled, A., Ozkara, K., Yu, T., Hong, M., and Park, Y.
\newblock Muonbp: Faster muon via block-periodic orthogonalization.
\newblock \emph{arXiv preprint arXiv:2510.16981}, 2025.

\bibitem[Kingma(2014)]{kingma2014adam}
Kingma, D.~P.
\newblock Adam: A method for stochastic optimization.
\newblock \emph{arXiv preprint arXiv:1412.6980}, 2014.

\bibitem[Klein et~al.(2009)Klein, Pluim, Staring, and Viergever]{klein2009adaptive}
Klein, S., Pluim, J.~P., Staring, M., and Viergever, M.~A.
\newblock Adaptive stochastic gradient descent optimisation for image registration.
\newblock \emph{International journal of computer vision}, 81\penalty0 (3):\penalty0 227--239, 2009.

\bibitem[Krizhevsky et~al.(2009)Krizhevsky, Hinton, et~al.]{krizhevsky2009learning}
Krizhevsky, A., Hinton, G., et~al.
\newblock Learning multiple layers of features from tiny images.
\newblock 2009.

\bibitem[Lan(2013)]{lan2013complexity}
Lan, G.
\newblock The complexity of large-scale convex programming under a linear optimization oracle.
\newblock \emph{arXiv preprint arXiv:1309.5550}, 2013.

\bibitem[Lau et~al.(2025)Lau, Long, and Su]{lau2025polargrad}
Lau, T. T.-K., Long, Q., and Su, W.
\newblock Polargrad: A class of matrix-gradient optimizers from a unifying preconditioning perspective.
\newblock \emph{arXiv preprint arXiv:2505.21799}, 2025.

\bibitem[Li \& Hong(2025)Li and Hong]{li2025note}
Li, J. and Hong, M.
\newblock A note on the convergence of muon.
\newblock \emph{arXiv preprint arXiv:2502.02900}, 2025.

\bibitem[Liu et~al.(2025)Liu, Su, Yao, Jiang, Lai, Du, Qin, Xu, Lu, Yan, et~al.]{liu2025muon}
Liu, J., Su, J., Yao, X., Jiang, Z., Lai, G., Du, Y., Qin, Y., Xu, W., Lu, E., Yan, J., et~al.
\newblock Muon is scalable for llm training.
\newblock \emph{arXiv preprint arXiv:2502.16982}, 2025.

\bibitem[Loshchilov \& Hutter(2017)Loshchilov and Hutter]{loshchilov2017decoupled}
Loshchilov, I. and Hutter, F.
\newblock Decoupled weight decay regularization.
\newblock \emph{arXiv preprint arXiv:1711.05101}, 2017.

\bibitem[Ly et~al.(2017)Ly, Marsman, Verhagen, Grasman, and Wagenmakers]{ly2017tutorial}
Ly, A., Marsman, M., Verhagen, J., Grasman, R.~P., and Wagenmakers, E.-J.
\newblock A tutorial on fisher information.
\newblock \emph{Journal of Mathematical Psychology}, 80:\penalty0 40--55, 2017.

\bibitem[Martens(2020)]{martens2020new}
Martens, J.
\newblock New insights and perspectives on the natural gradient method.
\newblock \emph{Journal of Machine Learning Research}, 21\penalty0 (146):\penalty0 1--76, 2020.

\bibitem[Martens \& Grosse(2015)Martens and Grosse]{martens2015optimizing}
Martens, J. and Grosse, R.
\newblock Optimizing neural networks with kronecker-factored approximate curvature.
\newblock In \emph{International conference on machine learning}, pp.\  2408--2417. PMLR, 2015.

\bibitem[Ollivier(2017)]{ollivier2017true}
Ollivier, Y.
\newblock True asymptotic natural gradient optimization.
\newblock \emph{arXiv preprint arXiv:1712.08449}, 2017.

\bibitem[Radford et~al.(2019)Radford, Wu, Child, Luan, Amodei, Sutskever, et~al.]{radford2019language}
Radford, A., Wu, J., Child, R., Luan, D., Amodei, D., Sutskever, I., et~al.
\newblock Language models are unsupervised multitask learners.
\newblock \emph{OpenAI blog}, 1\penalty0 (8):\penalty0 9, 2019.

\bibitem[Robbins \& Monro(1951)Robbins and Monro]{robbins1951stochastic}
Robbins, H. and Monro, S.
\newblock A stochastic approximation method.
\newblock \emph{The annals of mathematical statistics}, pp.\  400--407, 1951.

\bibitem[Roux et~al.(2007)Roux, Manzagol, and Bengio]{roux2007topmoumoute}
Roux, N., Manzagol, P.-A., and Bengio, Y.
\newblock Topmoumoute online natural gradient algorithm.
\newblock \emph{Advances in neural information processing systems}, 20, 2007.

\bibitem[Saad(2003)]{saad2003iterative}
Saad, Y.
\newblock \emph{Iterative methods for sparse linear systems}.
\newblock SIAM, 2003.

\bibitem[Shazeer \& Stern(2018)Shazeer and Stern]{shazeer2018adafactor}
Shazeer, N. and Stern, M.
\newblock Adafactor: Adaptive learning rates with sublinear memory cost.
\newblock In \emph{International Conference on Machine Learning}, pp.\  4596--4604. PMLR, 2018.

\bibitem[Shen et~al.(2025)Shen, Huang, Huang, Shen, and Zhang]{shen2025convergence}
Shen, W., Huang, R., Huang, M., Shen, C., and Zhang, J.
\newblock On the convergence analysis of muon.
\newblock \emph{arXiv preprint arXiv:2505.23737}, 2025.

\bibitem[Shrestha(2023)]{shrestha2023natural}
Shrestha, R.
\newblock Natural gradient methods: Perspectives, efficient-scalable approximations, and analysis.
\newblock \emph{arXiv preprint arXiv:2303.05473}, 2023.

\bibitem[Su(2025)]{su2025isotropic}
Su, W.
\newblock Isotropic curvature model for understanding deep learning optimization: Is gradient orthogonalization optimal?
\newblock \emph{arXiv preprint arXiv:2511.00674}, 2025.

\bibitem[Team et~al.(2023)Team, Anil, Borgeaud, Alayrac, Yu, Soricut, Schalkwyk, Dai, Hauth, Millican, et~al.]{team2023gemini}
Team, G., Anil, R., Borgeaud, S., Alayrac, J.-B., Yu, J., Soricut, R., Schalkwyk, J., Dai, A.~M., Hauth, A., Millican, K., et~al.
\newblock Gemini: a family of highly capable multimodal models.
\newblock \emph{arXiv preprint arXiv:2312.11805}, 2023.

\bibitem[Touvron et~al.(2023)Touvron, Lavril, Izacard, Martinet, Lachaux, Lacroix, Rozi{\`e}re, Goyal, Hambro, Azhar, et~al.]{touvron2023llama}
Touvron, H., Lavril, T., Izacard, G., Martinet, X., Lachaux, M.-A., Lacroix, T., Rozi{\`e}re, B., Goyal, N., Hambro, E., Azhar, F., et~al.
\newblock Llama: Open and efficient foundation language models.
\newblock \emph{arXiv preprint arXiv:2302.13971}, 2023.

\bibitem[Tveit et~al.(2025)Tveit, Remseth, and Skogvold]{tveit2025muon}
Tveit, A., Remseth, B., and Skogvold, A.
\newblock Muon optimizer accelerates grokking.
\newblock \emph{arXiv preprint arXiv:2504.16041}, 2025.

\bibitem[Vyas et~al.(2024)Vyas, Morwani, Zhao, Kwun, Shapira, Brandfonbrener, Janson, and Kakade]{vyas2024soap}
Vyas, N., Morwani, D., Zhao, R., Kwun, M., Shapira, I., Brandfonbrener, D., Janson, L., and Kakade, S.
\newblock Soap: Improving and stabilizing shampoo using adam.
\newblock \emph{arXiv preprint arXiv:2409.11321}, 2024.

\bibitem[Wang et~al.(2025)Wang, Zhang, Li, Du, Du, Pang, Yang, Hong, and Tan]{wang2025muon}
Wang, S., Zhang, F., Li, J., Du, C., Du, C., Pang, T., Yang, Z., Hong, M., and Tan, V.~Y.
\newblock Muon outperforms adam in tail-end associative memory learning.
\newblock \emph{arXiv preprint arXiv:2509.26030}, 2025.

\bibitem[Yu et~al.(2018)Yu, Wang, Shelhamer, and Darrell]{yu2018deep}
Yu, F., Wang, D., Shelhamer, E., and Darrell, T.
\newblock Deep layer aggregation.
\newblock In \emph{Proceedings of the IEEE conference on computer vision and pattern recognition}, pp.\  2403--2412, 2018.

\bibitem[Zhang et~al.(2024)Zhang, Chen, Li, Ding, Wu, Kingma, Ye, Luo, and Sun]{zhang2024adam}
Zhang, Y., Chen, C., Li, Z., Ding, T., Wu, C., Kingma, D.~P., Ye, Y., Luo, Z.-Q., and Sun, R.
\newblock Adam-mini: Use fewer learning rates to gain more.
\newblock \emph{arXiv preprint arXiv:2406.16793}, 2024.

\end{thebibliography}
\bibliographystyle{icml2026}

\newpage
\appendix
\onecolumn

\section{Preliminary Theorems and Lemmas on Matrix Analysis}

\begin{lemma}[Generalized von Neumann Trace Inequality (rectangular form)]
    \label{lem:tr-ineq}
Let $A,B \in \RR^{m\times n}$ be two given matrices, and let $p=\min\{m,n\}$. Suppose their singular values are $\sigma_1(A)\ge\sigma_2(A)\ge \cdots \ge \sigma_p(A)$ and $\sigma_1(B)\ge\sigma_2(B)\ge \cdots \ge \sigma_p(B)$. Then we have the following inequality
\begin{equation}
    |\tr(A^\top B)|\le\sum_{i=1}^{p}\sigma_i(A)\sigma_i(B).
\end{equation}
\end{lemma}

The above lemma can be found in the book \textit{Topics in Matrix analysis} \cite{horn1994topics}.

\begin{theorem}[H\"older inequality for Schatten norm]\label{thm:schatten_holder}
Let $A,B\in\mathbb{R}^{m\times n}$ and let $p,q\in[1,\infty]$ satisfy $\frac{1}{p}+\frac{1}{q}=1$
(with the convention $\frac{1}{\infty}=0$). Then
\begin{equation}\label{eq:schatten_holder_two}
\big|\langle A,B\rangle_F\big|
=\big|\mathrm{tr}(A^\top B)\big|
\;\le\;
\|A\|_{S_p}\,\|B\|_{S_q},
\end{equation}
where $\|\cdot\|_{S_p}$ denotes the Schatten-$p$ norm.
In particular, taking $(p,q)=(1,\infty)$ yields
\begin{equation}\label{eq:nuclear_spectral_duality}
\big|\mathrm{tr}(A^\top B)\big|
\;\le\;
\|A\|_{*}\,\|B\|_{\mathrm{op}},
\end{equation}
where $\|A\|_*=\|A\|_{S_1}$ is the nuclear norm and $\|B\|_{\mathrm{op}}=\|B\|_{S_\infty}$ is the spectral (operator) norm.
\end{theorem}

\begin{theorem}[Generalized H\"older inequality for Schatten norm]
\label{thm:schatten_holder_three}
Let $A\in\mathbb{R}^{m\times m}$, $X\in\mathbb{R}^{m\times n}$, and $B\in\mathbb{R}^{n\times n}$.
Let $p\in[1,\infty]$ and let $p_1,p_2,p_3\in[1,\infty]$ satisfy
$$\frac{1}{p}= \frac{1}{p_1}+\frac{1}{p_2}+\frac{1}{p_3}.$$
Then the following inequalities hold:
$$\|A X B\|_{S_p}\le \|A\|_{S_{P_{1}}}\,\|X\|_{S_{p_2}}\,\|B\|_{S_{p_3}}.
$$
In particular, taking $p=p_2=1$ and $p_1=p_3=\infty$ gives the commonly used nuclear-norm bound
\begin{equation}\label{eq:AXB_nuclear}
\|A X B\|_{*}
\le
\|A\|_{\mathrm{op}}\,\|X\|_{*}\,\|B\|_{\mathrm{op}}.
\end{equation}
\end{theorem}

\begin{lemma}[Equivalence of Frobenius and Nuclear norm]
\label{lem:eq-norm}
Let $A \in \RR^{m\times n}$, suppose $r$ is the rank of matrix $A$, then we have the following inequality:
\begin{equation}
    \|A\|_F \le \|A\|_* \le \sqrt{r}\|A\|_F.
\end{equation}
    
\end{lemma}

\section{Proof of Theorem~\ref{thm:opt-Kron-approx}}
\label{app:best-pq-thm}
\begin{proof}
We prove the claim for $P$, then the proof of $Q$ is analogous.

Fix $Q\in \SSS_{++}^n$. For simpler notation, we denote $A(Q):=\mathbb{E}[GQ^{-1}G^\top]$. Then the components of $\cJ(P,Q)$ depending on $P$ is
$$
J(P, Q)\;=\;\mathrm{tr}\left(P^{-1}A(Q)\right)
+\mu\,\mathrm{tr}(Q^{-1})\,\mathrm{tr}(P^{-1})
\;+\; n\log\det P.
$$
By taking the gradient of $\cJ$ w.r.t. $P$, we have
$$
\frac{\partial}{\partial P} J(P,Q) = -P^{-1}(A(Q)+ \mu\tr(Q^{-1})I_m)P^{-1}+nP^{-1}.
$$
Setting $\frac{\partial}{\partial P} J(P,Q)=0$, we have
$$
(A(Q)+ \mu\tr(Q^{-1})I_m)P^{-1} = nI_m.
$$
Therefore,
$$
P = \frac{1}{n}(A(Q)+ \mu\tr(Q^{-1})I_m) = \frac{1}{n}\EE[GQ^{-1}G^T] + \frac{\mu\tr(Q^{-1})}{n} I_m,
$$
which is exactly the claim in Theorem~\ref{thm:opt-Kron-approx}.

It remains to prove the uniqueness. Notice that the function
$P\mapsto \mathrm{tr}(P^{-1}S)+n\log\det P$ with $S\succ 0$ is strictly convex on $\SSS_{++}^m$. Since $\mu\tr(Q^{-1})>0$ when $\mu >0$, here $S = \mathrm{tr}\left(P^{-1}A(Q)\right)
+\mu\,\mathrm{tr}(Q^{-1})$ is positive definite. And hence $J(P,Q)$ is strictly convex w.r.t. $P$. Therefore, the stationary point for $J$ is exactly the unique minimizer. 

This complete the proof. 
\end{proof}

\section{Proof of Theorem~\ref{thm:update}}
\label{app:update-thm}

\begin{proof}[Proof of Theorem~\ref{thm:update}]
    The main idea of the proof is to transform the constraints into a standard spectral-norm ball. Let $\Phi:=P^{1/2}\Delta W Q^{1/2}$ be  the preconditioned update. Since $P$ and $Q$ are positive definite, the mapping from $\Delta W$ to $\Phi$ is bijective, and therefore $\Delta W = P^{-1/2}\Phi Q^{-1/2}$. Then the constraint of \eqref{eq:lmo} becomes $\|\Phi\|_2\le \eta$.

    By the cyclic property of trace, the objective of problem~\eqref{eq:lmo} becomes
    \begin{align}
        \langle G, \Delta W \rangle_F &= \tr\left( G^\top P^{-1/2} \Phi Q^{-1/2} \right) = \tr\left(Q^{-1/2}  G^\top P^{-1/2} \Phi \right)  \nonumber\\ &= \tr\left( \left( P^{-1/2} G Q^{-1/2} \right)^\top \Phi \right) = \langle \widetilde{G}, \Phi \rangle_F.
        \label{eq:tr-cyc}
    \end{align}
    
    Therefore, the problem~\eqref{eq:lmo} is equivalent to
    \begin{equation}
        \label{eq:lmo_new}
        \min_{\Phi \in \RR^{m\times n}}  \langle \widetilde{G}, \Phi \rangle_F  \quad \text{s.t.} \quad \|\Phi\|_2 \le \eta.
    \end{equation}

    \paragraph{Lower Bound.} Let $p=\min\{m,n\}$. Apply lemma~\ref{lem:tr-ineq} with $A = \widetilde{G}$ and $B=\Phi$, we have
    $$\langle \widetilde{G}, \Phi \rangle_F = \operatorname{tr}(\widetilde{G}^\top \Phi) \geq - \sum_{i=1}^{p} \sigma_i(\widetilde{G}) \sigma_i(\Phi),$$
    where the singular values $\sigma_i$'s of both $\widetilde{G}$ and $\Phi$ are arranged in descending order as in lemma~\ref{lem:tr-ineq}. Notice that the spectral norm (i.e. $\|\cdot\|_2$) of a matrix equals to its largest singular value, therefore, 
    $$\sigma_i(\Phi)\le \sigma_1(\Phi)=\|\Phi\|_2\le\eta, \quad \forall i=1, \cdots,p.$$
    Therefore, $$\langle\widetilde{G},\Phi \rangle_F\ge-\sum_{i=1}^p \sigma_i(\widetilde{G})\cdot\eta=-\eta\sum_{i=1}^p\sigma_i(\widetilde{G})=-\eta\|\widetilde{G}\|_*,$$
    which means that $-\eta\|\widetilde{G}\|_*$ is a lower bound on the minimum of the problem~\ref{eq:lmo_new}. 

    \paragraph{Feasibility and Achievability.} We consider the SVD $\widetilde{G}=U\Sigma V^\top$, where $\Sigma=\diag(\sigma_1(\widetilde{G}),\cdots,\sigma_r(\widetilde{G}))$ with $r=\rank(\widetilde{G})\le p$. Define
    \begin{equation}
        \Phi^*:=-\eta UV^\top.
    \end{equation}
    Since the matrix $UV^\top$ has singular values equal to 1, so $\|UV^\top\|_2=1$. So $\|\Phi^*\|_2=\eta\|UV^\top\|_2=\eta$, and thus $\Phi^*$ is feasible.

    Now we consider the objective value of problem~\ref{eq:lmo_new} given $\Phi=\Phi^*$:
    \begin{align*}
        \langle \widetilde{G}, \Phi^\star \rangle_F &= -\eta \langle U \Sigma V^\top, U V^\top \rangle_F = -\eta \tr((U \Sigma V^\top)^\top U V^\top)\\
        &=-\eta\tr(V\Sigma U^\top UV^\top)=-\eta\tr(V\Sigma V^\top)\\
        &=-\eta\sum_{i=1}^r \sigma_i(\widetilde{G})=-\eta\|\widetilde{G}\|_*.
    \end{align*}
    Therefore, $\|\Phi^*\|=UV^\top$ attains the global lower bound and is optimal for problem~\ref{eq:lmo_new}. Thus the solution for problem~\ref{eq:lmo} is:
    $$\Delta W^\star = P^{-1/2} \Phi^\star Q^{-1/2} = -\eta P^{-1/2} U V^\top Q^{-1/2}=-\eta P^{-1/2}\Polar(P^{-1/2}GQ^{-1/2})Q^{-1/2}.$$
\end{proof}

\section{Proof of the Lemmas for Convergence Analysis}
\label{app:cvg-lemma}

\begin{proof}[Proof of Lemma~\ref{lem:cvg-lem}]
We summarize the update rules for Algorithm~\ref{alg:kl} as follows:
$$G_t \leftarrow \nabla_W\mathcal{L}(W_{t-1}),\quad
\widetilde{G}_t := P_t^{-1/2}G_tQ_t^{-1/2},\quad
M_t := \beta M_{t-1} + (1-\beta)\widetilde{G}_t,$$
$$\Delta W_t := P_t^{-1/2}\,\mathrm{Polar}(M_t)\,Q_t^{-1/2},\quad
W_t := W_{t-1} - \eta\,\Delta W_t.$$
Denote the whitened true gradient
\begin{equation}\label{eq:tilde_nabla_def}
\widetilde{\nabla}\mathcal{L}(W_{t-1})
:= P_t^{-1/2}\,\nabla\mathcal{L}(W_{t-1})\,Q_t^{-1/2}.
\end{equation}
Apply the Lipschitz smoothness condition with respect to $W_t$ and $W_{t-1}$, we have
\begin{align}
\mathcal{L}(W_t) &\le \mathcal{L}(W_{t-1})
+\langle \nabla\mathcal{L}(W_{t-1}),\,W_t-W_{t-1}\rangle_F
+\frac{L}{2}\|W_t-W_{t-1}\|_F^2 \nonumber\\
&= \mathcal{L}(W_{t-1}) -\eta\langle \nabla\mathcal{L}(W_{t-1}),\,\Delta W_t\rangle_F
+\frac{L\eta^2}{2}\|\Delta W_t\|_F^2.
\label{eq:smooth-cvg-lem}
\end{align}

We first consider the term $\langle \nabla\cL(W_{t-1}),\Delta W_{t}\rangle_F$. Denote $O_t = \Polar(M_t)$, so $\Delta W_t=P_t^{-1/2}O_tQ_t^{-1/2}$. By the cyclic property of trace (as in equation~\eqref{eq:tr-cyc}), we have
\begin{align}
\langle \nabla\mathcal{L}(W_{t-1}),\,\Delta W_t\rangle_F
&=
\left\langle \nabla\mathcal{L}(W_{t-1}),\,P_t^{-1/2}O_tQ_t^{-1/2}\right\rangle_F \nonumber\\
&=
\left\langle P_t^{-1/2}\nabla\mathcal{L}(W_{t-1})Q_t^{-1/2},\,O_t\right\rangle_F \nonumber\\
&=
\left\langle \widetilde{\nabla}\mathcal{L}(W_{t-1}),\,O_t\right\rangle_F \nonumber\\
&=
\left\langle M_t,O_t\right\rangle_F+\left\langle \widetilde{\nabla}\mathcal{L}(W_{t-1})-M_t,\,O_t\right\rangle_F.
\label{eq:align_whitened}
\end{align}
Since $O_t=\mathrm{Polar}(M_t)$, we have $\|O_t\|_{\mathrm{op}}=\|O_t\|_2=1$ and $\langle M_t,U_t\rangle_F=\|M_t\|_*$. By H\"older inequality~\ref{thm:schatten_holder}, we have
\begin{equation}\label{eq:holder}
\left\langle \widetilde{\nabla}\mathcal{L}(W_{t-1})-M_t,\,O_t\right\rangle_F
\ge
-\big\|\widetilde{\nabla}\mathcal{L}(W_{t-1})-M_t\big\|_*\cdot \|O_t\|_{\mathrm{op}}
=
-\big\|\widetilde{\nabla}\mathcal{L}(W_{t-1})-M_t\big\|_*.
\end{equation}
Moreover, by the triangle inequality,
\begin{equation}
\label{eq:tri_M}
\|M_t\|_*\ge \big\|\widetilde{\nabla}\mathcal{L}(W_{t-1})\big\|_*-\big\|\widetilde{\nabla}\mathcal{L}(W_{t-1})-M_t\big\|_*.
\end{equation}
Combining \eqref{eq:align_whitened}--\eqref{eq:tri_M} yields
\begin{align}
\label{eq:align_lower}
\left\langle \widetilde{\nabla}\mathcal{L}(W_{t-1}),O_t\right\rangle_F
&\ge
\|M_t\|_*-\big\|\widetilde{\nabla}\mathcal{L}(W_{t-1})-M_t\big\|_* \nonumber\\
&\ge
\big\|\widetilde{\nabla}\mathcal{L}(W_{t-1})\big\|_*
-2\big\|\widetilde{\nabla}\mathcal{L}(W_{t-1})-M_t\big\|_*.
\end{align}
Plug \eqref{eq:align_lower} into \eqref{eq:smooth-cvg-lem}, we have 
\begin{align}
\label{eq:cvg-lem-new}
\mathcal{L}(W_t)
&\le
\mathcal{L}(W_{t-1})
-\eta\left\langle \widetilde{\nabla}\mathcal{L}(W_{t-1}),O_t\right\rangle_F
+\frac{L\eta^2}{2}\|\Delta W_t\|_F^2 \nonumber\\
&\le
\mathcal{L}(W_{t-1})
-\eta\big\|\widetilde{\nabla}\mathcal{L}(W_{t-1})\big\|_*
+2 \eta\big\|\widetilde{\nabla}\mathcal{L}(W_{t-1})-M_t\big\|_*
+\frac{L\eta^2}{2}\|\Delta W_t\|_F^2.
\end{align}

For the second term in the RHS of \eqref{eq:cvg-lem-new},
\begin{equation}
    \label{eq:cvg-lem-1}
    \big\|\widetilde{\nabla}\mathcal{L}(W_{t-1})\big\|_* = \big\|P_t^{-1/2}\nabla\mathcal{L}(W_{t-1})Q_t^{-1/2}\big\|_* 
    \le \big\|P_t^{-1/2}\big\|_2 \cdot\big\|\nabla\mathcal{L}(W_{t-1})\big\|_* \cdot \big\|P_t^{-1/2}\big\|_2
 \end{equation}

For the third term in the RHS of \eqref{eq:cvg-lem-new}, subtract and add a $\widetilde{G}_t$ and apply triangle inequality, we have
$$\big\|\widetilde{\nabla}\mathcal{L}(W_{t-1})-M_t\big\|_*
\le \big\|\widetilde{\nabla}\mathcal{L}(W_{t-1})-\widetilde{G}_t\big\|_* + \big\|\widetilde{G}_t-M_t\big\|_*,$$
where 
\begin{align}
\label{eq:cvg-lem-2}
    \big\|\widetilde{\nabla}\mathcal{L}(W_{t-1})-\widetilde{G}_t\big\|_* 
    &= \big\|P_t^{-1/2}\left(\nabla\mathcal{L}(W_{t-1})-G_t\right)Q_t^{-1/2}\big\|_* \nonumber\\
    &\le \big\|P_t^{-1/2}\big\|_2 \cdot\big\|\nabla\mathcal{L}(W_{t-1})-G_t\big\|_* \cdot \big\|P_t^{-1/2}\big\|_2,
\end{align}
according to the Generalized H\"older inequality for Schatten norm in Theorem~\ref{thm:schatten_holder_three}.

For the fourth term in the RHS of \eqref{eq:cvg-lem-new}, we have 
\begin{equation}
\label{eq:cvg-lem-3}
    \|\Delta W_t\|_F=\|P_t^{-1/2}O_tQ_t^{-1/2}\|_F
\le \|P_t^{-1/2}\|_2\|O_t\|_F\|Q_t^{-1/2}\|_2.
\end{equation}
Since $O_t=\mathrm{Polar}(M_t)$ is a partial isometry, we have $\|O_t\|_F^2=\mathrm{rank}(O_t)=\mathrm{rank}(M_t)\le \min(m,n)$.

Plug \eqref{eq:cvg-lem-1}--\eqref{eq:cvg-lem-3} into \eqref{eq:cvg-lem-new}, we derive the final formula in Lemma~\ref{lem:cvg-lem}.

\end{proof}

\begin{lemma}[Positive definiteness of $P$ and $Q$]
\label{lem:pd-pq}
    For each $t$, the $P_t$ and $Q_t$ derived from Algorithm~\ref{alg:kl} are positive definite.
\end{lemma}

\begin{proof}[Proof of Lemma~\ref{lem:pd-pq}]
We prove the lemma by mathematical induction. By initialization, clearly $P_0=I_m\succ 0$ and $Q_0=I_n\succ 0$.

Assume $P_{t-1}\succ 0$ and $Q_{t-1}\succ 0$ for some $t\ge 1$.
We need to show that $P_t\succ 0$ and $Q_t\succ 0$.
By the algorithm we have
$$L_t \;=\; \frac{1}{n}G_t Q_{t-1}^{-1} G_t^\top \;+\; \mu \frac{\tr(P_{t-1})}{m} I_m.$$

Since $Q_{t-1}\succ 0$, we have $Q_{t-1}^{-1}\succ 0$. Hence $\frac{1}{n}G_t Q_{t-1}^{-1} G_t^\top \succeq 0.$

Also, $P_{t-1}\succ 0$ implies $\tr(P_{t-1})>0$, and with $\mu>0$ the term
$\mu \frac{\tr(P_{t-1})}{m} I_m$ is strictly positive definite. Therefore, for any $x\neq 0$,
\[
x^\top L_t x
= \frac{1}{n}(G_t^\top x)^\top Q_{t-1}^{-1}(G_t^\top x) + \mu \frac{\tr(P_{t-1})}{m}\|x\|_2^2
\ge \mu \frac{\tr(P_{t-1})}{m}\|x\|_2^2 > 0,
\]
which shows $L_t\succ 0$.

Since $P_{t-1}\succ 0$ and $L_t\succ 0$ and $\gamma\in[0,1)$, therefore, $\widetilde{P}_t \;=\; \gamma P_{t-1} + (1-\gamma)L_t$ is positive definite.

Finally, Algorithm~\ref{alg:kl} normalizes
$P_t \;=\; \mathrm{sym}\!\left(\frac{m}{\tr(\widetilde{P}_t)}\,\widetilde{P}_t\right)$. Because $\widetilde{P}_t$ is symmetric positive definite and $\frac{m}{\tr(\widetilde{P}_t)}>0$,
the scaled matrix $\frac{m}{\tr(\widetilde{P}_t)}\widetilde{P}_t$ is symmetric positive definite.
Moreover, $\mathrm{sym}(\cdot)$ preserves symmetry and does not affect positive definiteness here.
Thus $P_t$ is positive definite.

Similarly, we can show that $Q_t$ is positive definite as well, and by induction, the positive-definiteness of $P$ and $Q$ extend to all iteration $t$'s.

\end{proof}

\begin{lemma}[Upper bound of $K_{PQ}(t)$]
    \label{lem:Kpq-bdd}
In Algorithm~\ref{alg:kl}, suppose there exists $G>0$ such that for all $t$, $\|G_t\|_F \le G$. Moreover, assume $\mu>G^2/(mn)$. Then there exist constant $\overline{r},\overline{s}$ such that for all $t$,
$$\|P_t^{-1}\|_2 \le\overline{r}, \quad \|Q_t^{-1}\|_2\le \overline{s},$$
and thus 
\begin{equation}
    K_{PQ}(t)=\big\|P_t^{-1/2}\big\|_2\big\|Q_t^{-1/2}\big\|_2\le\sqrt{\overline{r}\ \overline{s}}=:\overline{K},
\end{equation}
which means $K_{PQ}(t)$ is uniformly upper-bounded.
\end{lemma}

\begin{proof}[Proof of Lemma~\ref{lem:Kpq-bdd}]
According to the update rules in Algorithm~\ref{alg:kl},
$$P_t = \sym\!\left(\frac{m}{\tr(\widetilde P_t)}\widetilde P_t\right),\qquad
Q_t = \sym\!\left(\frac{n}{\tr(\widetilde Q_t)}\widetilde Q_t\right).$$
Since $\sym(\cdot)$ preserves trace, we have, for all $t\ge 0$,
\begin{equation}
\tr(P_t)=\tr\left(\frac{m}{\tr(\widetilde P_t)}\widetilde P_t\right)=m,\qquad \tr(Q_t)=\tr\left(\frac{m}{\tr(\widetilde Q_t)}\widetilde Q_t\right)=n.
\label{eq:trace_invariant}
\end{equation}
Hence 
$$\mu \frac{\mathrm{tr}(P_{t-1})}{m}I_m = \mu I_m,\qquad
\mu \frac{\mathrm{tr}(Q_{t-1})}{n}I_n = \mu I_n.$$
Moreover, by Lemma~\ref{lem:pd-pq}, $Q_{t-1}^{-1}\succ 0$, we have $\frac{1}{n}G_tQ_{t-1}^{-1}G_t^\top\succeq 0$, hence
\begin{equation}
L_t=\frac{1}{n}G_tQ_{t-1}^{-1}G_t^\top+\mu I_m \succeq \mu I_m.
\label{eq:L_lower}
\end{equation}
Similarly, we have
\begin{equation}
R_t=\frac{1}{m}G_t^\top P_t^{-1}G_t+\mu I_n \succeq \mu I_n.
\label{eq:R_lower}
\end{equation}

Denote $\lambda_{\min}(A)$ as the smallest eigenvalue of matrix $A$. From $\widetilde P_t=\gamma P_{t-1}+(1-\gamma)L_t$ and \eqref{eq:L_lower}, by Weyl's eigenvalue inequality, 
\begin{equation}
\lambda_{\min}(\widetilde P_t)\ge \gamma\,\lambda_{\min}(P_{t-1})+(1-\gamma)\lambda_{\min}(L_t)\ge (1-\gamma)\mu.
\label{eq:Ptilde_lmin}
\end{equation}
Since $P_t=\frac{m}{\mathrm{tr}(\widetilde P_t)}\widetilde P_t$, we have
$$P_t^{-1}=\frac{\mathrm{tr}(\widetilde P_t)}{m}\,\widetilde P_t^{-1},
\quad\text{hence}\quad
\|P_t^{-1}\|_2=\frac{\mathrm{tr}(\widetilde P_t)}{m}\cdot \frac{1}{\lambda_{\min}(\widetilde P_t)}.$$
It remains to upper bound $\mathrm{tr}(\widetilde P_t)$. By \eqref{eq:trace_invariant},
\begin{equation}
\mathrm{tr}(\widetilde P_t)=\gamma\,\mathrm{tr}(P_{t-1})+(1-\gamma)\mathrm{tr}(L_t)=\gamma m+(1-\gamma)\mathrm{tr}(L_t).
\label{eq:tr_Ptilde}
\end{equation}
We also have,
$$\mathrm{tr}(L_t)=\frac{1}{n}\mathrm{tr}(G_tQ_{t-1}^{-1}G_t^\top)+\mu m
=\frac{1}{n}\mathrm{tr}(Q_{t-1}^{-1}G_t^\top G_t)+\mu m.$$

By von Neumann trace inequality, for positive definite matrix $A, B \in \SSS^n$, we have
\begin{equation}
    \label{eq:tr-ineq2}
    \tr(AB)\le \sum_{i=1}^n \lambda_i(A)\lambda_i(B)\le \sum_{i=1}^n \lambda_{\max}(A)\lambda_i(B)=\lambda_{\max}(A)\sum_{i=1}^{n}\lambda_i(B) = \|A\|_2 \tr(B).
\end{equation}

Since $\mathrm{tr}(G_t^\top G_t)=\|G_t\|_F^2$, we obtain
$$\mathrm{tr}(Q_{t-1}^{-1}G_t^\top G_t)\le \|Q_{t-1}^{-1}\|_2\,\|G_t\|_F^2.$$
By assuming $\|G_t\|_F\le G$ for all $t$,
$$\mathrm{tr}(L_t)\le \mu m+\frac{G^2}{n}\|Q_{t-1}^{-1}\|_2.$$

Substituting this bound into \eqref{eq:tr_Ptilde} and combining with \eqref{eq:Ptilde_lmin} yields
\begin{equation}
\|P_t^{-1}\|_2 \le
\frac{\gamma m+(1-\gamma)\left(\mu m+\frac{G^2}{n}\|Q_{t-1}^{-1}\|_2\right)}{m(1-\gamma)\mu}
= a+b\,\|Q_{t-1}^{-1}\|_2,
\label{eq:r_rec}
\end{equation}
where
\begin{equation}
a:=\frac{\gamma+(1-\gamma)\mu}{(1-\gamma)\mu},
\qquad
b:=\frac{G^2}{mn\,\mu}.
\label{eq:ab_def}
\end{equation}

Similarly, from $\widetilde Q_t=\gamma Q_{t-1}+(1-\gamma)R_t$ and \eqref{eq:R_lower},
\begin{equation}
\lambda_{\min}(\widetilde Q_t)\ge (1-\gamma)\mu, \qquad \mathrm{tr}(\widetilde Q_t)=\gamma n+(1-\gamma)\mathrm{tr}(R_t).
\label{eq:Qtilde_lmin}
\end{equation}
Moreover,
$$\mathrm{tr}(R_t)=\frac{1}{m}\mathrm{tr}(G_t^\top P_t^{-1}G_t)+\mu n
=\frac{1}{m}\mathrm{tr}(P_t^{-1}G_tG_t^\top)+\mu n
\le \frac{1}{m}\|P_t^{-1}\|_2\,\|G_t\|_F^2+\mu n
\le \mu n+\frac{G^2}{m}\|P_t^{-1}\|_2.$$
Therefore, by the same normalization argument as for $P_t$ and \eqref{eq:Qtilde_lmin}, we have
\begin{equation}
\|Q_t^{-1}\|_2
\le
\frac{\gamma n+(1-\gamma)\left(\mu n+\frac{G^2}{m}\|P_t^{-1}\|_2\right)}{n(1-\gamma)\mu}
=
a+b\,\|P_t^{-1}\|_2,
\label{eq:s_rec}
\end{equation}
with the same $a,b$ as in \eqref{eq:ab_def}.

For simplicity, we denote $r_t:=\|P_t^{-1}\|_2$ and $s_t:=\|Q_t^{-1}\|_2$. Equations \eqref{eq:r_rec}--\eqref{eq:s_rec} give
$$r_t\le a+b\,s_{t-1}, \qquad s_t\le a+b\,r_t.$$
Combining them yields
$$s_t \le a+b(a+b\,s_{t-1})=a(1+b)+b^2 s_{t-1}.$$
If $b<1$, then $b^2<1$ and iterating the above recursion gives
$$s_t \le a(1+b)\sum_{k=0}^{t-1}b^{2k}+b^{2t}s_0
\le \frac{a(1+b)}{1-b^2}+s_0.$$
Since $Q_0=I_n$, we have $s_0=\|Q_0^{-1}\|_2=1$. By defining
$$\overline s:=\frac{a(1+b)}{1-b^2}+1,$$
we obtain $s_t\le \overline s$ for all $t$. Plugging into $r_t\le a+b s_{t-1}$ yields
$$r_t\le a+b\,\overline s =:\overline r,\qquad \forall t.$$

Finally,
$$K_{PQ}(t)=\|P_t^{-1/2}\|_2\|Q_t^{-1/2}\|_2
=\sqrt{\|P_t^{-1}\|_2}\,\sqrt{\|Q_t^{-1}\|_2}
=\sqrt{r_t s_t}
\le \sqrt{\overline r\,\overline s}
=: \overline K,$$
which proves that $K_{PQ}(t)$ is uniformly upper-bounded.

\end{proof}

\begin{lemma}[Lower bound of $K_{PQ}$]
    \label{lem:Kpq_lower}
    Under the same condition as in Lemma~\ref{lem:Kpq-bdd}, there exist a constant $\underline{K}$ such that
    \begin{equation}
        K_{PQ}(t)\ge \underline{K}.
    \end{equation}
    for all iteration $t$.
\end{lemma}

\begin{proof}[Proof of Lemma~\ref{lem:Kpq_lower}]
Since $P_t\succ 0$ by Lemma~\ref{lem:pd-pq}, following the similar analysis on $P_t$, we have
$$\|P_t^{-1/2}\|_2 \;=\; \sqrt{\lambda_{\max}(P_t^{-1})} \;=\;\frac{1}{\sqrt{\lambda_{\min}(P_t)}}.$$
Moreover, since $P_t \succ 0$ is positive semidefinite, its largest eigenvalue is bounded by its trace:
$$\lambda_{\max}(P_t)\;\le\;\tr(P_t)\;=\;m.$$
It follows that
$$\lambda_{\min}(P_t)\;\le\;m
\qquad\Longrightarrow\qquad
\|P_t^{-1/2}\|_2 \;=\;\frac{1}{\sqrt{\lambda_{\min}(P_t)}}\;\ge\;\frac{1}{\sqrt{m}}.$$
Similarly, since $Q_t\succ 0$ and using $\tr(Q_t)=n$, we obtain
$$\|Q_t^{-1/2}\|_2 \;\ge\; \frac{1}{\sqrt{n}}.$$
Combining the two inequalities yields
$$K_{PQ}(t)
\;=\;\|P_t^{-1/2}\|_2\,\|Q_t^{-1/2}\|_2
\;\ge\;\frac{1}{\sqrt{m}}\cdot \frac{1}{\sqrt{n}}
\;=\;\frac{1}{\sqrt{mn}}=:\underline{K}.$$
This completes the proof.

\end{proof}

\begin{lemma}[Lipschitz bound for inverse square root]
\label{lem:lipschitz-invsqrt}
Let $A,B\in\RR^{m\times m}$ be positive definite matrices. Assume that $A \succeq \underline{c}\,I_m$ and $B \succeq \underline{c}\,I_m$
for some constant $\underline{c}>0$. Then
$$\bigl\|A^{-1/2}-B^{-1/2}\bigr\|_2
\;\le\;
\frac{1}{2\,\underline{c}^{3/2}}\;\|A-B\|_2 .$$
\end{lemma}

\begin{proof}[Proof of Lemma~\ref{lem:lipschitz-invsqrt}]
For any positive definite matrix $X$, the following identity holds:
\begin{equation}
X^{-1/2}=
\frac{1}{\pi}\int_{0}^{\infty} t^{-1/2}\,(X+tI_m)^{-1}\,\mathrm{d}t .
\label{eq:invsqrt-integral}
\end{equation}
This follows from the identity
$x^{-1/2}=\frac{1}{\pi}\int_{0}^{\infty} t^{-1/2}(x+t)^{-1}\,\mathrm{d}t$.

Therefore, we have
\begin{align}
A^{-1/2}-B^{-1/2}
&=
\frac{1}{\pi}\int_{0}^{\infty} t^{-1/2}\Bigl[(A+tI_m)^{-1}-(B+tI_m)^{-1}\Bigr]\mathrm{d}t .
\label{eq:diff-integral}
\end{align}
Moreover, the resolvent identity gives
\begin{equation}
(A+tI_m)^{-1}-(B+tI_m)^{-1}
=
(A+tI_m)^{-1}(B-A)(B+tI_m)^{-1}.
\label{eq:resolvent}
\end{equation}
Substituting \eqref{eq:resolvent} into \eqref{eq:diff-integral} yields
\begin{equation}
A^{-1/2}-B^{-1/2}
=
\frac{1}{\pi}\int_{0}^{\infty} t^{-1/2}(A+tI_m)^{-1}(B-A)(B+tI_m)^{-1}\,\mathrm{d}t .
\label{eq:diff-resolvent-integral}
\end{equation}

Taking spectral norm of both sides and applying submultiplicativity, we get
\begin{align}
\bigl\|A^{-1/2}-B^{-1/2}\bigr\|_2
&\le
\frac{1}{\pi}\int_{0}^{\infty}
t^{-1/2}\,\bigl\|(A+tI_m)^{-1}\bigr\|_2\,\|A-B\|_2\,\bigl\|(B+tI_m)^{-1}\bigr\|_2
\,\mathrm{d}t .
\label{eq:norm-bound-int}
\end{align}
Since $A\succeq \underline{c}I_m$, we have $A+tI_m\succeq (\underline{c}+t)I_m$, hence
\begin{equation}
\bigl\|(A+tI_m)^{-1}\bigr\|_2 \le \frac{1}{\underline{c}+t},
\qquad
\bigl\|(B+tI_m)^{-1}\bigr\|_2 \le \frac{1}{\underline{c}+t}.
\label{eq:resolvent-bound}
\end{equation}
Plugging \eqref{eq:resolvent-bound} into \eqref{eq:norm-bound-int} gives
\begin{equation}
\bigl\|A^{-1/2}-B^{-1/2}\bigr\|_2
\le
\frac{\|A-B\|_2}{\pi}\int_{0}^{\infty} t^{-1/2}\,\frac{1}{(\underline{c}+t)^2}\,\mathrm{d}t .
\label{eq:int-reduced}
\end{equation}

For simplicity, we change the variable by $t=\underline{c}u$. Then $\mathrm{d}t=\underline{c}\,\mathrm{d}u$,
$t^{-1/2}=\underline{c}^{-1/2}u^{-1/2}$, and $(\underline{c}+t)^2=\underline{c}^2(1+u)^2$.
Therefore,
\begin{align}
\int_{0}^{\infty} t^{-1/2}\,\frac{1}{(\underline{c}+t)^2}\,\mathrm{d}t
&=
\underline{c}^{-3/2}\int_{0}^{\infty}\frac{u^{-1/2}}{(1+u)^2}\,\mathrm{d}u .
\label{eq:int-changevar}
\end{align}
Clearly the integral is in form of Beta-function. By the Beta-function identity
$$\int_{0}^{\infty}\frac{u^{a-1}}{(1+u)^{a+b}}\,\mathrm{d}u=
\mathrm{B}(a,b)=\frac{\Gamma(a)\Gamma(b)}{\Gamma(a+b)},$$
with $a=\tfrac12$ and $b=\tfrac32$, we obtain
\begin{equation}
\int_{0}^{\infty}\frac{u^{-1/2}}{(1+u)^2}\,\mathrm{d}u
=\mathrm{B}\!\left(\tfrac12,\tfrac32\right)
=\frac{\Gamma(\tfrac12)\Gamma(\tfrac32)}{\Gamma(2)}
=\frac{\sqrt{\pi}\cdot \tfrac12\sqrt{\pi}}{1}
=\frac{\pi}{2}.
\label{eq:beta-eval}
\end{equation}
Combining \eqref{eq:int-changevar} and \eqref{eq:beta-eval} yields
\begin{equation}
\int_{0}^{\infty} t^{-1/2}\,\frac{1}{(\underline{c}+t)^2}\,\mathrm{d}t
=
\underline{c}^{-3/2}\cdot \frac{\pi}{2}.
\label{eq:int-final}
\end{equation}

Plug \eqref{eq:int-final} into \eqref{eq:int-reduced}, we obtain
\[
\bigl\|A^{-1/2}-B^{-1/2}\bigr\|_2
\le
\frac{\|A-B\|_2}{\pi}\cdot \underline{c}^{-3/2}\cdot \frac{\pi}{2}
=
\frac{1}{2\,\underline{c}^{3/2}}\|A-B\|_2,
\]
which completes the proof.
    
\end{proof}

\begin{lemma}[Preconditioner drift bound by step size]
    \label{lem:drift-PQ-bdd}
Suppose the EMA hyperparameter $\gamma$ and the learning rate $\eta$ in Algorithm~\ref{alg:kl} satisfies $1-\gamma\le c_\gamma \eta$, for some constant $c_\gamma>0$. Under the same condition as in Lemma\ref{lem:Kpq-bdd}, there exist constants $C_P$ and $C_Q$ such that 
\begin{equation}
    \label{eq:lem_drift_bdd}
    \big\|P_t^{-1/2} - P_{t-1}^{-1/2}\big\|_2 \le C_P\eta, \qquad \big\|Q_t^{-1/2} - Q_{t-1}^{-1/2}\big\|_2 \le C_Q\eta.
\end{equation}
\end{lemma}

\begin{proof}[Proof of Lemma~\ref{lem:drift-PQ-bdd}]

We first prove that there exist positive $\underline{p}$ and $\overline{p}$ such that $\underline{p}I \preceq P_t \preceq \overline{p} I$ for all $t$.

From Lemma~\ref{lem:Kpq-bdd} and Lemma~\ref{lem:Kpq_lower} we know there exist positive $a_P$ and $b_P$ such that $a_P \le \big\|P_t^{-1/2}\big\|_2 \le b_P$ for all $t$. In particular, $a_P \le 1/\sqrt{m}$ and $b_P \ge \sqrt{\overline{r}}$.
Also we have
$$\big\|P_t^{-1/2}\big\|_2 = \lambda_{\max}(P_t^{-1/2}) = \frac{1}{\sqrt{\lambda_{\min}(P_t)}} \le b_P.$$
Therefore,
$$\lambda_{\min}(P_t) \ge \frac{1}{b_P^2}.$$
Take $\underline{p}=1/b_P^2$. For the upper bound, we can simply take $\overline{p}=m$, since $\lambda_{\max}(P_t)\le \tr(P_t)=m$. Therefore, we have
\begin{equation}
    \lambda_{\min}(P_t) \ge \underline{p}, \qquad \lambda_{\max}(P_t)\le \overline{p}.
\end{equation}
Hence $\underline{p}I \preceq P_t \preceq \overline{p} I$.

Similarly, we conclude that there exist positive $\underline{q}$ and $\overline{q}$ such that $\underline{q}I \preceq Q_t \preceq \overline{q}I$ for all $t$.

According to Lemma~\ref{lem:lipschitz-invsqrt}, we have
$$\big\|P_t^{-1/2}-P_{t-1}^{-1/2}\big\|_2 \le \frac{1}{2\underline{p}^{3/2}}\|P_t-P_{t-1}\|_2.
$$
Therefore, it suffice to prove 
\begin{equation}
    \label{eq:Pt-drift-goal}
    \|P_t - P_{t-1}\|_2 \le c_P \eta
 \end{equation}
for some constant $c_P$, after which setting $C_P=c_p/(2\underline{p}^{3/2})$ will give the result.

Recall that from Algorithm~\ref{alg:kl}, the update rules are
$$\widetilde{P}_t = \gamma P_{t-1} + (1-\gamma)L_t,
\qquad
P_t = \frac{m}{\tr(\widetilde{P}_t)}\,\widetilde{P}_t.$$
Define $\alpha_t := \frac{m}{\tr(\widetilde{P}_t)} \;>\;0$, then $P_t = \alpha_t \widetilde{P}_t$.
Then we have
$$P_t - P_{t-1}= \alpha_t \widetilde{P}_t - P_{t-1} \nonumber= (\alpha_t-1)P_{t-1} + \alpha_t(\widetilde{P}_t - P_{t-1}).$$
Taking operator norms on both sides, we get
\begin{equation}
\label{eq:Pt-norm-split}
\|P_t-P_{t-1}\|_2
\le |\alpha_t-1|\,\|P_{t-1}\|_2 + \alpha_t\,\|\widetilde{P}_t-P_{t-1}\|_2.
\end{equation}
By the definition of EMA in Algorithm~\ref{alg:kl},
\begin{equation}\label{eq:Ptilde-minus-P}
\widetilde{P}_t - P_{t-1}
= \gamma P_{t-1} + (1-\gamma)L_t - P_{t-1}
= (1-\gamma)(L_t - P_{t-1}).
\end{equation}
Hence
\begin{equation}\label{eq:Ptilde-bound}
\|\widetilde{P}_t - P_{t-1}\|_2
\le (1-\gamma)\big(\|L_t\|_2 + \|P_{t-1}\|_2\big).
\end{equation}
Using $\|P_{t-1}\|_2 \le \overline{p}$, it remains to bound $\|L_t\|_2$.
Recall
$$L_t = \frac{1}{n}G_tQ_{t-1}^{-1}G_t^\top + \mu\frac{\tr(P_{t-1})}{m}I_m
= \frac{1}{n}G_tQ_{t-1}^{-1}G_t^\top + \mu I_m.$$
Therefore,
\begin{align}
\|L_t\|_2
&\le \frac{1}{n}\|G_t\|_2^2\,\|Q_{t-1}^{-1}\|_2 + \mu
\le \frac{1}{n}\,G_2^2\cdot \frac{1}{\underline{q}} + \mu
=: \overline{L}_P. \label{eq:Lt-op-bound}
\end{align}
Plugging \eqref{eq:Lt-op-bound} into \eqref{eq:Ptilde-bound} yields
\begin{equation}\label{eq:Ptilde-final}
\|\widetilde{P}_t - P_{t-1}\|_2
\le (1-\gamma)\big(\overline{L}_P + \overline{p}\big).
\end{equation}
Under the timescale condition $1-\gamma \le c_\gamma \eta$, we obtain
\begin{equation}\label{eq:Ptilde-Oeta}
\|\widetilde{P}_t - P_{t-1}\|_2
\le c_\gamma\big(\overline{L}_P + \overline{p}\big)\eta.
\end{equation}

Now we consider $\alpha_t$ and $|\alpha_t - 1|$. Since $\widetilde{P}_t = \gamma P_{t-1} + (1-\gamma)L_t \succeq \gamma P_{t-1} \succeq \gamma \underline{p}I_m$, therefore, $\tr(\widetilde{P}_t) \ge \gamma \underline{p}\,m$.
Hence we have
\begin{equation}\label{eq:alpha-upper}
\alpha_t = \frac{m}{\tr(\widetilde{P}_t)} \le \frac{1}{\gamma \underline{p}}.
\end{equation}
For $|\alpha_t-1|$, we have
\begin{equation}\label{eq:alpha-minus-1}
|\alpha_t-1|
= \left|\frac{m-\tr(\widetilde{P}_t)}{\tr(\widetilde{P}_t)}\right|
\le \frac{|m-\tr(\widetilde{P}_t)|}{\gamma \underline{p}\,m}.
\end{equation}
Taking trace of $\widetilde{P}_t = \gamma P_{t-1} + (1-\gamma)L_t$ gives
$$\tr(\widetilde{P}_t) = \gamma \tr(P_{t-1}) + (1-\gamma)\tr(L_t) = \gamma m + (1-\gamma)\tr(L_t),$$
hence
\begin{equation}\label{eq:trace-diff}
m-\tr(\widetilde{P}_t) = (1-\gamma)\big(m-\tr(L_t)\big).
\end{equation}
Combining \eqref{eq:alpha-minus-1}--\eqref{eq:trace-diff} yields
\begin{equation}\label{eq:alpha-minus-1-bound}
|\alpha_t-1|
\le \frac{1-\gamma}{\gamma \underline{p}\,m}\,|m-\tr(L_t)|.
\end{equation}
It remains to find an upper bound for $|m-\tr(L_t)|$.
Using $L_t = \frac{1}{n}G_tQ_{t-1}^{-1}G_t^\top + \mu I_m$,
$$\tr(L_t) = \frac{1}{n}\tr(Q_{t-1}^{-1}G_t^\top G_t) + \mu m
\le \frac{1}{n}\|Q_{t-1}^{-1}\|_2\tr(G_t^\top G_t) + \mu m.$$
Moreover, $\|Q_{t-1}^{-1}\|_2 \le 1/\underline{q}$ and $\tr(G_t^\top G_t)=\|G_t\|_F^2 \le \min(m,n)\|G_t\|_2^2 \le \min(m,n)G_2^2$.
Thus there exists a constant $c_L>0$ such that $\tr(L_t)\le c_L m$, and consequently
\begin{equation}\label{eq:m-trLt}
|m-\tr(L_t)| \le m+\tr(L_t) \le (1+c_L)m.
\end{equation}
Plugging \eqref{eq:m-trLt} into \eqref{eq:alpha-minus-1-bound} gives
\begin{equation}\label{eq:alpha-O1mgamma}
|\alpha_t-1|
\le \frac{1+c_L}{\gamma \underline{p}}\,(1-\gamma)
\le \frac{1+c_L}{\gamma \underline{p}}\,c_\gamma\,\eta.
\end{equation}

Substitute \eqref{eq:Ptilde-Oeta}, \eqref{eq:alpha-upper}, and \eqref{eq:alpha-O1mgamma} into \eqref{eq:Pt-norm-split}, and use $\|P_{t-1}\|_2\le \overline{p}$:
\begin{align*}
\|P_t-P_{t-1}\|_2
&\le |\alpha_t-1|\,\|P_{t-1}\|_2 + \alpha_t\,\|\widetilde{P}_t-P_{t-1}\|_2\\
&\le \left(\frac{1+c_L}{\gamma \underline{p}}c_\gamma\,\eta\right)\overline{p}
\;+\;
\left(\frac{1}{\gamma \underline{p}}\right)\left(c_\gamma(\overline{L}_P+\overline{p})\eta\right)\\
&= \left[
\frac{(1+c_L)c_\gamma\,\overline{p}}{\gamma \underline{p}}
+\frac{c_\gamma(\overline{L}_P+\overline{p})}{\gamma \underline{p}}
\right]\eta
=: c_P\,\eta,
\end{align*}
which proves \eqref{eq:Pt-drift-goal}. Therefore, we finish the boundedness of the drift for preconditioner $P_t$. i.e. 
$$\big\|P_t^{-1/2}-P_{t-1}^{-1/2}\big\|_2
\le \frac{c_P}{2\underline{p}^{3/2}}\eta
= C_P\eta,$$
as claimed. 

Similarly, replacing $(m,P,L,\underline{p},\overline{p})$ by $(n,Q,R,\underline{q},\overline{q})$, we can prove that there exist a constant $C_Q$ such that
$$\big\|Q_t^{-1/2}-Q_{t-1}^{-1/2}\big\|_2
\le C_Q\eta,$$
which finishs the whole proof.

\end{proof}

\begin{lemma}
\label{lem:ema-update-bd}
    Given the update rules in Algorithm~\ref{alg:kl}, the following inequality holds:
\begin{equation}
\label{eq:ema-update-bd}
    \sum_{t=1}^{T} \|\widetilde{G}_t - M_t\|_* \leq \frac{\beta}{1-\beta} \sum_{t=2}^{T} \|\widetilde{G}_t - \widetilde{G}_{t-1}\|_* + \frac{\beta}{1-\beta} \sup_{1 \leq t \leq T} \|\widetilde{G}_t\|_*.
\end{equation}
\end{lemma}

\begin{proof}[Proof of Lemma~\ref{lem:ema-update-bd}]
We first claim that the closed form for $M_t$ is given by
\begin{equation}
    \label{eq:closed-form-M_t}
    M_t = (1 - \beta) \sum_{s=1}^{t} \beta^{t-s} \widetilde{G}_s.
\end{equation}
We prove the claim by induction as follows.
Clear equation~\eqref{eq:closed-form-M_t} holds for $t=1$. Suppose \eqref{eq:closed-form-M_t} also holds for $t-1$, then
\begin{align*}
M_t &= \beta M_{t-1} + (1 - \beta)\widetilde{G}_t \\
&= \beta \left( (1-\beta) \sum_{s=1}^{t-1} \beta^{t-1-s} \widetilde{G}_s \right) + (1-\beta)\widetilde{G}_t \\
&= (1-\beta) \sum_{s=1}^{t-1} \beta^{t-s} \widetilde{G}_s + (1-\beta)\beta^0 \widetilde{G}_t \\
&= (1-\beta) \sum_{s=1}^{t} \beta^{t-s} \widetilde{G}_s,
\end{align*}
which proves the claim.

Therefore, we have
\begin{equation}
\begin{aligned}
    \widetilde{G}_t-M_t &= \widetilde{G}_t-(1-\beta)\sum_{s=1}^t \beta^{t-s} \widetilde{G}_s\\
    &= \beta \widetilde{G}_t - (1-\beta)\sum_{s=1}^{t-1} \beta^{t-s} \widetilde{G}_s.
\end{aligned}
\end{equation}

Notice that by geometric-sum,
$$
(1-\beta)\sum_{s=1}^{t-1} \beta^{t-s} = (1-\beta)\sum_{\hat{s}=1}^{t-1}\beta^{\hat{s}} = \beta(1-\beta^{t-1})=\beta-\beta^t.
$$
Therefore,
$$
\beta \widetilde{G}_t = (1-\beta)\sum_{s=1}^{t-1}\beta^{t-s}\widetilde{G}_t + \beta^t \widetilde{G}_t.
$$
Hence
\begin{equation}
    \label{eq:G_t-M_t-new}
    \widetilde{G}_t - M_t = (1 - \beta) \sum_{s=1}^{t-1} \beta^{t-s} (\widetilde{G}_t - \widetilde{G}_s) + \beta^t \widetilde{G}_t.
\end{equation}

Take nuclear norm on both sides and apply triangle inequality, we get
\begin{equation}
\|\widetilde{G}_t - M_t\|_*
\le
(1-\beta)\sum_{s=1}^{t-1}\beta^{\,t-s}\|\widetilde{G}_t-\widetilde{G}_s\|_*
+\beta^t\|\widetilde{G}_t\|_*.
\label{eq:lag-norm-1}
\end{equation}
For each $1\le s\le t-1$, telescoping yields
\begin{equation}
\|\widetilde{G}_t-\widetilde{G}_s\|_*
\le
\sum_{j=s+1}^{t}\|\widetilde{G}_j-\widetilde{G}_{j-1}\|_*.
\label{eq:telescoping}
\end{equation}
Plugging \eqref{eq:telescoping} into \eqref{eq:lag-norm-1} gives
\begin{align}
\|\widetilde{G}_t - M_t\|_*
&\le
(1-\beta)\sum_{s=1}^{t-1}\beta^{\,t-s}\sum_{j=s+1}^{t}\|\widetilde{G}_j-\widetilde{G}_{j-1}\|_*
+\beta^t\|\widetilde{G}_t\|_* \notag\\
&=
\sum_{j=2}^{t}\Big((1-\beta)\sum_{s=1}^{j-1}\beta^{\,t-s}\Big)\|\widetilde{G}_j-\widetilde{G}_{j-1}\|_*
+\beta^t\|\widetilde{G}_t\|_* \notag\\
&\le
\sum_{j=2}^{t}\beta^{\,t-j+1}\,\|\widetilde{G}_j-\widetilde{G}_{j-1}\|_*
\;+\;\beta^t\|\widetilde{G}_t\|_*,
\label{eq:swap-sums}
\end{align}

where the third line use the fact that, for each $j\le t$,
$$(1-\beta)\sum_{s=1}^{j-1}\beta^{\,t-s}
=(1-\beta)\beta^{\,t-j+1}\sum_{r=0}^{j-2}\beta^r
\le \beta^{\,t-j+1}.$$

Summing \eqref{eq:swap-sums} over $t=1,\dots,T$,
\begin{align}
\sum_{t=1}^{T}\|\widetilde{G}_t - M_t\|_*
&\le
\sum_{t=2}^{T}\sum_{j=2}^{t}\beta^{\,t-j+1}\|\widetilde{G}_j-\widetilde{G}_{j-1}\|_*
+
\sum_{t=1}^{T}\beta^{t}\|\widetilde{G}_t\|_* \notag\\
&=
\sum_{j=2}^{T}\Big(\sum_{t=j}^{T}\beta^{\,t-j+1}\Big)\|\widetilde{G}_j-\widetilde{G}_{j-1}\|_*
+
\sum_{t=1}^{T}\beta^{t}\|\widetilde{G}_t\|_*.
\label{eq:sum-swap-2}
\end{align}
Notice that, for the geometric tail,
$$\sum_{t=j}^{T}\beta^{\,t-j+1} \le \sum_{r=1}^{\infty}\beta^r = \frac{\beta}{1-\beta}.$$
Thus the first term in \eqref{eq:sum-swap-2} is bounded by
$$\frac{\beta}{1-\beta}\sum_{j=2}^{T}\|\widetilde{G}_j-\widetilde{G}_{j-1}\|_*.$$
For the residual term, since $\|\widetilde{G}_t\|_*\le \sup_{1\le s\le T}\|\widetilde{G}_s\|_*$, we obtain
$$\sum_{t=1}^{T}\beta^{t}\|\widetilde{G}_t\|_*
\le
\Big(\sum_{t=1}^{\infty}\beta^{t}\Big)\sup_{1\le s\le T}\|\widetilde{G}_s\|_*
=
\frac{\beta}{1-\beta}\sup_{1\le s\le T}\|\widetilde{G}_s\|_*.$$
This complete the proof.
    
\end{proof}

\begin{lemma}[Variance reduction w.r.t. batch size]
\label{lem:var-bd}
Define the gradient sampling noise
\begin{align*}
    N_t :=& G_t - \nabla \mathcal{L}(W_{t-1}) \\
    =& \frac{1}{B} \sum_{i=1}^{B} \left( \nabla \ell(W_{t-1}; \xi_{t,i}) - \nabla \mathcal{L}(W_{t-1}) \right).
\end{align*}
Under Assumption~\ref{ass:unbias-grad}--\ref{ass:bv}, we have
\begin{equation}
    \mathbb{E}[\|N_t\|_F^2 \mid \mathcal{F}_{t-1}] \leq \frac{\sigma^2}{B}.
\end{equation}
\end{lemma}

\begin{proof}
Let $Z_{t,i} := \nabla \ell(W_{t-1}; \xi_{t,i}) - \nabla \mathcal{L}(W_{t-1})$. Then $N_t = \frac{1}{B} \sum_{i=1}^B Z_{t,i}$ and $\mathbb{E}[Z_{t,i} \mid \mathcal{F}_{t-1}] = 0$. Using independence across $i$ and expanding the squared Frobenius norm:
\begin{align*}
    \mathbb{E}[\|N_t\|_F^2 \mid \mathcal{F}_{t-1}] &= \mathbb{E}\bigg[\bigg\|\frac{1}{B} \sum_{i=1}^B Z_{t,i}\bigg\|_F^2 \bigg| \mathcal{F}_{t-1}\bigg] \\
    &= \frac{1}{B^2}\mathbb{E}\bigg[ \sum_{i=1}^B \|Z_{t,i}\|_F^2 + \sum_{i \neq j} \langle Z_{t,i}, Z_{t,j}\rangle_F \bigg| \mathcal{F}_{t-1}\bigg].
\end{align*}
The cross terms vanish:
\[
    \mathbb{E}[\langle Z_{t,i}, Z_{t,j}\rangle_F \mid \mathcal{F}_{t-1}] = \Big\langle \mathbb{E}[Z_{t,i} \mid \mathcal{F}_{t-1}], \mathbb{E}[Z_{t,j} \mid \mathcal{F}_{t-1}] \Big\rangle_F = 0 \quad (i \neq j),
\]
so
\[
    \mathbb{E}[\|N_t\|_F^2 \mid \mathcal{F}_{t-1}] = \frac{1}{B^2} \sum_{i=1}^B \mathbb{E}[\|Z_{t,i}\|_F^2 \mid \mathcal{F}_{t-1}] \leq \frac{1}{B^2} \cdot B\sigma^2 = \frac{\sigma^2}{B}.
\]
\end{proof}

\section{Proof of the convergence theorem}
\label{app:cvg-thm}

\begin{proof}[Proof of Theorem~\ref{thm:main_convergence_kl_muon_bigO}]
Taking the expectation on both sides of the inequality in Lemma~\ref{lem:cvg-lem}, we have
\begin{align}
\EE[\cL(W_t)]
&\le
\EE[\cL(W_{t-1})]
-\eta\,\EE\!\left[K_{PQ}(t)\|\nabla \cL(W_{t-1})\|_*\right]
+2\eta\,\EE\!\left[K_{PQ}(t)\|\nabla \cL(W_{t-1})-G_t\|_*\right] \notag\\
&\quad
+2\eta\,\EE\!\left[\|\widetilde{G}_t-M_t\|_*\right]
+\frac{L\eta^2}{2}\min(m,n)\,\EE\!\left[K_{PQ}^2(t)\right].
\label{eq:exp_one_step_v2}
\end{align}
Summing \eqref{eq:exp_one_step_v2} from $t=1$ to $T$ yields
\begin{align}
\EE[\cL(W_T)]
&\le
\cL(W_0)
-\eta\sum_{t=1}^T \EE\!\left[K_{PQ}(t)\|\nabla \cL(W_{t-1})\|_*\right]
+2\eta\sum_{t=1}^T \EE\!\left[K_{PQ}(t)\|\nabla \cL(W_{t-1})-G_t\|_*\right] \notag\\
&\quad
+2\eta\sum_{t=1}^T \EE\!\left[\|\widetilde{G}_t-M_t\|_*\right]
+\frac{L\eta^2}{2}\min(m,n)\sum_{t=1}^T \EE\! \left[K_{PQ}^2(t)\right],
\label{eq:telescope_v2}
\end{align}
where the objective terms telescope as
$\sum_{t=1}^T(\EE[\cL(W_t)]-\EE[\cL(W_{t-1})])=\EE[\cL(W_T)]-\cL(W_0)$.

By Assumption~\ref{ass:LB}, the objective function is lower-bounded by $\cL_*$, so $\EE[\cL(W_T)]\ge \cL_\star$. Rearranging \eqref{eq:telescope_v2} and dividing both sides by $\eta T$ gives

\begin{equation}
\begin{aligned}
\frac{1}{T}\sum_{t=1}^T \EE\!\left[K_{PQ}(t)\|\nabla \cL(W_{t-1})\|_*\right]
&\le
\frac{\cL(W_0)-\cL_\star}{\eta T}
+\bmark{\frac{2}{T}\sum_{t=1}^T \EE\!\left[K_{PQ}(t)\|\nabla \cL(W_{t-1})-G_t\|_*\right]}{Stochastic gradient noise term}
\\
&+
\bmark{\frac{2}{T}\sum_{t=1}^T \EE\!\left[\|\widetilde{G}_t-M_t\|_*\right]}{Momentum tracking error}
+\bmark{\frac{L\eta}{2}\min(m,n)\cdot \frac{1}{T}\sum_{t=1}^T \EE\!\left[K_{PQ}(t)^2\right]}{Quadratic term (smoothness penalty)}.
\end{aligned}
\label{eq:avg_start_v2} 
\end{equation}

Now we consider the terms in the right-head side of \eqref{eq:avg_start_v2} separately.
\paragraph{Quadratic term.} According to Lemma~\ref{lem:Kpq-bdd}, the term $K_{PQ}(t)$ is upper-bounded by $\overline{K}$, implying $K_{PQ}^2(t) \le\overline{K}^2$. Therefore,
\begin{equation}
\label{eq:quad-rs}
    \frac{1}{T}\sum_{t=1}^T \EE\!\left[K_{PQ}(t)^2\right]
\le
\overline{K}^2.
\end{equation}

\paragraph{Noise term.} By Lemma~\ref{lem:eq-norm}, we know $\|A\|_*\le \sqrt{\min(m,n)}\,\|A\|_F$ for all matrix $A$. Applying the upper bound of variance in Lemma~\ref{lem:var-bd}, we have
\begin{align}
\EE\!\left[K_{PQ}(t)\|\nabla \cL(W_{t-1})-G_t\|_*\right]
&\le
\overline{K}\,\EE\!\left[\|\nabla \cL(W_{t-1})-G_t\|_*\right] \notag\\
&\le
\overline{K}\sqrt{\min(m,n)}\,
\EE\!\left[\|\nabla \cL(W_{t-1})-G_t\|_F\right] \notag\\
&\le
\overline{K}\sqrt{\min(m,n)}\,
\sqrt{\EE\!\left[\|\nabla \cL(W_{t-1})-G_t\|_F^2\right]} \notag\\
&\le
\overline{K}\sqrt{\min(m,n)}\,
\frac{\sigma}{\sqrt{B}},
\label{eq:noise_bound_single_v2}
\end{align}
Averaging \eqref{eq:noise_bound_single_v2} over $t$ gives
\begin{equation}
\label{eq:noise-rs}
\frac{2}{T}\sum_{t=1}^T \EE\!\left[K_{PQ}(t)\|\nabla \cL(W_{t-1})-G_t\|_*\right]
\le
\frac{2\overline{K}\sqrt{\min(m,n)}\,\sigma}{\sqrt{B}}.
\end{equation}

\paragraph{Momentum tracking term.} By Lemma~\ref{lem:ema-update-bd}, we have
\begin{equation}
    \label{eq:momentum-tracking-formula}
    \sum_{t=1}^{T} \|\widetilde{G}_t - M_t\|_* \leq \frac{\beta}{1-\beta} \sum_{t=2}^{T} \|\widetilde{G}_t - \widetilde{G}_{t-1}\|_* + \frac{\beta}{1-\beta} \sup_{1 \leq t \leq T} \|\widetilde{G}_t\|_*.
\end{equation}
Now we need to scale-up the term $\|\widetilde{G}_t - \widetilde{G}_{t-1}\|_*$. By adding and subtracting intermediate terms, we have
\begin{align*}
\widetilde{G}_t-\widetilde{G}_{t-1}
=\;&(P_t^{-1/2}-P_{t-1}^{-1/2})\,G_t\,Q_t^{-1/2} \\
&+ P_{t-1}^{-1/2}\,(G_t-G_{t-1})\,Q_t^{-1/2}\\
&+ P_{t-1}^{-1/2}\,G_{t-1}\,(Q_t^{-1/2}-Q_{t-1}^{-1/2}). 
\end{align*}
Using Lemma~\ref{thm:schatten_holder_three} and triangle inequality, we obtain
\begin{align}
\|\widetilde{G}_t-\widetilde{G}_{t-1}\|_*
\le\;& \|P_t^{-1/2}-P_{t-1}^{-1/2}\|_2\;\|G_t\|_*\;\|Q_t^{-1/2}\|_2 \label{eq:inc1}\\
&+\|P_{t-1}^{-1/2}\|_2\;\|G_t-G_{t-1}\|_*\;\|Q_t^{-1/2}\|_2 \label{eq:inc2}\\
&+\|P_{t-1}^{-1/2}\|_2\;\|G_{t-1}\|_*\;\|Q_t^{-1/2}-Q_{t-1}^{-1/2}\|_2. \label{eq:inc3}
\end{align}

By Lemma~\ref{lem:Kpq-bdd} and Lemma~\ref{lem:drift-PQ-bdd}, there exist constants $\overline{K},C_P,C_Q>0$ such that for all $t$,

$$K_{PQ}(t):=\|P_t^{-1/2}\|_2\|Q_t^{-1/2}\|_2 \le \overline{K},$$

$$\|P_t^{-1/2}-P_{t-1}^{-1/2}\|_2 \le C_P\,\eta, \qquad
\|Q_t^{-1/2}-Q_{t-1}^{-1/2}\|_2 \le C_Q\,\eta.$$

Since we assume a uniform bound on the stochastic gradient magnitude in nuclear norm:
\begin{equation}
\|G_t\|_* \le G_* \quad \text{for all } t. \label{eq:Gstar}
\end{equation}
Then \eqref{eq:inc1} and \eqref{eq:inc3} are bounded by

$$\|P_t^{-1/2}-P_{t-1}^{-1/2}\|_2\;\|G_t\|_*\;\|Q_t^{-1/2}\|_2
\le C_P\eta\cdot \overline{K}G_*,$$
and
$$\|P_{t-1}^{-1/2}\|_2\;\|G_{t-1}\|_*\;\|Q_t^{-1/2}-Q_{t-1}^{-1/2}\|_2
\le C_Q\eta\cdot \overline{K}G_*.$$
Moreover, the middle term \eqref{eq:inc2} satisfies
$$\|P_{t-1}^{-1/2}\|_2\;\|G_t-G_{t-1}\|_*\;\|Q_t^{-1/2}\|_2
\le \overline{K}\,\|G_t-G_{t-1}\|_*.$$
Hence, combining the three parts, we have
\begin{equation}
\|\widetilde{G}_t-\widetilde{G}_{t-1}\|_*
\le (C_P+C_Q)\eta\cdot \overline{K}G_* \;+\; \overline{K}\,\|G_t-G_{t-1}\|_*.
\label{eq:inc_reduced}
\end{equation}

Since $\mathcal{L}$ is $L$-smooth in Frobenius norm by Assumption~\ref{ass:L-smooth}, applying $\|A\|_* \le \sqrt{\min(m,n)}\|A\|_F$, we have
\begin{align}
\|G_t-G_{t-1}\|_*
&\le \sqrt{\min(m,n)}\,\|G_t-G_{t-1}\|_F \nonumber\\
&\le L\sqrt{\min(m,n)}\,\|W_{t-1}-W_{t-2}\|_F. \label{eq:Gdiff_smooth}
\end{align}
Hence it remains to bound $\|\Delta W_{t-1}\|_F$.
Recall $\Delta W_{t-1} = P_{t-1}^{-1/2}\,\mathrm{Polar}(M_{t-1})\,Q_{t-1}^{-1/2}$.
By Theorem~\ref{thm:schatten_holder_three} and $\|\mathrm{Polar}(M_{t-1})\|_F \le \sqrt{\min(m,n)}$, we get
\begin{equation}
\|\Delta W_{t-1}\|_F
\le \|P_{t-1}^{-1/2}\|_2 \,\|\mathrm{Polar}(M_{t-1})\|_F\,\|Q_{t-1}^{-1/2}\|_2
\le \sqrt{\min(m,n)}\,K_{PQ}(t-1)
\le \sqrt{\min(m,n)}\,\overline{K}.
\label{eq:DeltaW_bound}
\end{equation}
Plugging \eqref{eq:DeltaW_bound} into \eqref{eq:Gdiff_smooth}, we obtain
\begin{equation}
\|G_t-G_{t-1}\|_* \le L\sqrt{\min(m,n)}\cdot \eta\cdot \sqrt{\min(m,n)}\,\overline{K}
= L\,\eta\,\min(m,n)\,\overline{K}.
\label{eq:Gdiff_final}
\end{equation}
Combining \eqref{eq:inc_reduced} and \eqref{eq:Gdiff_final} yields an increment bound:
\begin{equation}
\|\widetilde{G}_t-\widetilde{G}_{t-1}\|_*
\le \eta\Big[(c_P+c_Q)\overline{K}G_* \;+\; L\,\min(m,n)\,\overline{K}^2\Big].
\label{eq:Gtilde_increment_Oeta}
\end{equation}

Plugging \eqref{eq:Gtilde_increment_Oeta} and $\sup_t\|\widetilde{G}_t\|_* \le \overline{K}G_*$ into \eqref{eq:momentum-tracking-formula} yields
\begin{equation*}
\sum_{t=1}^T \|\widetilde{G}_t - M_t\|_*
\le
\frac{\beta}{1-\beta}(T-1)\eta\Big[(C_P+C_Q)\overline{K}G_* + L\,\min(m,n)\,\overline{K}^2\Big]
\;+\;
\frac{\beta}{1-\beta}\,\overline{K}G_*.
\end{equation*}
Dividing by $T$ and taking expectation,
\begin{equation}
\label{eq:moment-rs}
\frac{1}{T} \sum_{t=1}^{T} \mathbb{E}\bigl[\|\widetilde{G}_t - M_t\|_*\bigr]
\;\le\;
\frac{\beta}{1-\beta}\,\eta\Bigl[(c_P+c_Q)\,\overline{K}\,G_* \;+\; L\,\min(m,n)\,\overline{K}^{\,2}\Bigr]
\;+\;
\frac{\beta}{1-\beta}\,\frac{\overline{K}\,G_*}{T}.
\end{equation}

\paragraph{Final upper bound}
Plugging \eqref{eq:quad-rs}, \eqref{eq:noise-rs}, and \eqref{eq:moment-rs} into \eqref{eq:avg_start_v2}, we get
\begin{equation}
\label{eq:final_bound}
\begin{aligned}
\frac{1}{T}\sum_{t=1}^T \EE\!\left[K_{PQ}(t)\,\|\nabla \cL(W_{t-1})\|_*\right]
\;\le\;&
\frac{\cL(W_0)-\cL_\star}{\eta T}
\;+\;\frac{2\beta}{1-\beta}\,\frac{\overline K\,G_*}{T}
\;+\;\frac{2\overline{K}\sqrt{\min(m,n)}\,\sigma}{\sqrt{B}}
\\
&+\eta\Biggl[
\frac{L}{2}\,\min(m,n)\,\overline K^{\,2}
+\frac{2\beta}{1-\beta}\Bigl((C_P+C_Q)\,\overline K\,G_* + L\,\min(m,n)\,\overline K^{\,2}\Bigr)
\Biggr].
\end{aligned}
\end{equation}

For simplicity, we denote $R:=\cL(W_0)-\cL_\star$, $C_{dr} = C_P+C_Q$ and $r:=\min(m,n)$. By Lemma~\ref{lem:Kpq_lower}, there exist $\underline{K}>0$ be such that $K_{PQ}(t)\ge \underline{K}$ for all $t$.
Then for each $t$,
$$K_{PQ}(t)\,\|\nabla \cL(W_{t-1})\|_*
\;\ge\;
\underline{K}\,\|\nabla \cL(W_{t-1})\|_*.$$
Taking expectation and 
averaging over $t=1,\dots,T$ gives
\begin{equation}
\label{eq:lhs_lower_via_K}
\frac{1}{T}\sum_{t=1}^T \EE\!\left[K_{PQ}(t)\,\|\nabla \cL(W_{t-1})\|_*\right]
\;\ge\;
\underline{K}\cdot \frac{1}{T}\sum_{t=1}^T \EE\!\left[\|\nabla \cL(W_{t-1})\|_*\right].
\end{equation}

Combining \eqref{eq:lhs_lower_via_K} with \eqref{eq:final_bound}, we obtain
\begin{equation}
\label{eq:grad_avg_final}
\begin{aligned}
\frac{1}{T}\sum_{t=1}^T \EE\!\left[\|\nabla \cL(W_{t-1})\|_*\right]
\;&\le\;
\frac{1}{\underline K}\Biggl[
\frac{R}{\eta T}
+\frac{2\beta}{1-\beta}\,\frac{\overline K\,G_*}{T}
+\frac{2\overline{K}\sqrt{r}\,\sigma}{\sqrt{B}}
+\eta\!\left(
\frac{L}{2}\,r\,\overline K^{\,2}
+\frac{2\beta}{1-\beta}\Bigl(C_{dr}\,\overline K\,G_* + L\,r\,\overline K^{\,2}\Bigr)
\right)
\Biggr].\\
&=\frac{R}{\underline{K}\,\eta\,T}
+\frac{2\beta}{1-\beta}\,\frac{\overline K\,G_*}{\underline{K}\,T}
+\frac{2\overline{K}}{\underline{K}}\frac{\sigma\sqrt{r}}{\sqrt{B}} 
+\frac{2\beta}{1-\beta}\,\frac{C_{dr}\,\overline K\,G_*}{\underline{K}}\eta
+\frac{1+3\beta}{2-2\beta}\frac{L\overline K^{\,2}}{\underline{K}}\,r\eta.
\end{aligned}
\end{equation}

In particular, if we take $\eta = C/\sqrt{T}$ for any constant $C$, we have 
\begin{equation*}
\label{eq:grad_avg_final_eta_C_over_sqrtT}
\frac{1}{T}\sum_{t=1}^T \EE\!\left[\|\nabla \cL(W_{t-1})\|_*\right]
\le\frac{R}{\underline{K}\,C\,\sqrt{T}}
+\frac{2\beta}{1-\beta}\,\frac{\overline K\,G_*}{\underline{K}\,T}
+\frac{2\overline{K}}{\underline{K}}\frac{\sigma\sqrt{r}}{\sqrt{B}}
+\frac{2\beta}{1-\beta}\,\frac{C_{dr}\,\overline K\,G_*}{\underline{K}}\frac{C}{\sqrt{T}}
+\frac{1+3\beta}{2-2\beta}\frac{\overline K^{\,2}}{\underline{K}}\,rL\frac{C}{\sqrt{T}}.
\end{equation*}

Therefore,
\begin{equation}
\label{eq:rate_bigO_concrete}
\frac{1}{T}\sum_{t=1}^T \EE\!\left[\|\nabla \cL(W_{t-1})\|_*\right]
\;=\;
\cO\!\left(\frac{R + rL + G_*}{\sqrt{T}} +\frac{G_*}{T}+\frac{\sigma\sqrt{r}}{\sqrt{B}}\right).
\end{equation}

This finishes the proof.

\end{proof}


\end{document}